\documentclass[onefignum,onetabnum]{siamonline171218}
\usepackage{silence}
\usepackage[utf8]{inputenc} 
\usepackage[T1]{fontenc}    
\usepackage{color}
\usepackage{hyperref}       
\usepackage{url}            
\usepackage{makecell}
\usepackage{amsmath,amssymb,algorithm,algorithmic}
\usepackage{booktabs}       
\usepackage{amsfonts}       
\usepackage{nicefrac}       
\usepackage{microtype}      
\usepackage{lipsum,bm}
\usepackage{graphicx}
\usepackage{ragged2e}
\usepackage{multirow}
\usepackage{multicol}
\usepackage{enumitem}
\setlist[enumerate]{leftmargin=.4in}
\setlist[itemize]{leftmargin=.3in}
\def\etal{\textit{et al. }}
\newcommand{\han}[1]{{\color{black}{#1}}}

\newsiamremark{remark}{Remark}
\newsiamremark{hypothesis}{Hypothesis}
\crefname{hypothesis}{Hypothesis}{Hypotheses}
\newsiamthm{claim}{Claim}

\headers{QIS : Interactive Segmentation via Quasi-Conformal Mappings}{H. Zhang, D. Zhang and L.M. Lui}

\title{QIS : Interactive Segmentation via Quasi-Conformal Mappings
\thanks{Submitted to the editors DATE.
}}

\author{Han Zhang \thanks{Department of Mathematics, City University of Hong Kong and Hong Kong Centre for Cerebro-Cardiovascular Health EngineeringCity University of Hong Kong, Hong Kong, China. (\email{hzhang863-c@my.cityu.edu.hk}).}
\and Daoping Zhang \thanks{School of Mathematical Sciences and LPMC, Nankai University, Tianjin, China (\email{daopingzhang@nankai.edu.cn}).}
\and Lok Ming Lui \thanks{Department of Mathematics, Chinese University of Hong Kong, Hong Kong, China (\email{lmluic@math.cuhk.edu.hk}).}
}

\begin{document}
\maketitle
\begin{abstract}
Interactive segmentation allows users to provide meaningful input to guide the segmentation process. However, an important problem in interactive segmentation lies in determining how to incorporate minimal yet meaningful user guidance into the segmentation model. In this paper, we propose the quasi-conformal interactive segmentation (QIS) model, which incorporates user input in the form of positive and negative clicks. Users mark a few pixels belonging to the object region as positive clicks, indicating that the segmentation model should include a region around these clicks. Conversely, negative clicks are provided on pixels belonging to the background, instructing the model to exclude the region near these clicks from the segmentation mask. By solving our proposed theoretical supported model, the segmentation mask is obtained by deforming a template mask with the same topology as the object of interest using a quasiconformal mapping. This approach makes each user input effectively used and helps to avoid topological errors in the segmentation results. We provide a thorough theoretical analysis of the proposed model for its ability to include or exclude regions of interest or disinterest based on the user's indication. To evaluate the performance of QIS, we conduct experiments on synthesized images, medical images, and natural images. The results demonstrate the efficacy of our proposed method.
\end{abstract}


\section{Introduction}

Interactive segmentation leverages the collaborative strengths of humans and machines, allowing users to provide interactive input to guide the segmentation process\cite{rother2004grabcut,rada2014variational}. Through techniques such as mouse clicks or scribbles, users actively participate in refining the segmentation and providing valuable feedback. As a human-involved technique, the segmented results of interactive segmentation become more versatile and believable in various scenarios. For example, in clinical medicine, accurately delineating diseased regions is a common requirement. However, due to differences in imaging device configurations, direct outputs from a non-interactive model may not yield satisfactory results for images captured with various devices. Improvements and corrections are often necessary to achieve superior segmentation. Instead of overhauling algorithms to suit the unique setup of each imaging device, an interactive segmentation model, capable of accommodating specific user guidance, offers an economical and readily deployable solution. This approach instills greater confidence in clinical professionals as they can actively participate in the segmentation process.

An efficient interactive segmentation model should minimize the need for user guidance. As such, three key principles should be followed for designing an effective interactive segmentation model: (1)Effective interaction response: Each new interaction input should result in a meaningful segmentation correction. For example, if a user clicks to include a specific region, the model should at least encompass the local area around the click. Otherwise, the click becomes ineffective, making the correction process cumbersome and unproductive. Despite its intuitive nature, many existing methods fail to achieve this, as highlighted in our experimental section. (2)Adaptability: A robust interactive segmentation model should be easily adaptable to various image modalities and types. In contrast, learning-based methods often require extensive training data with well-labeled ground truth, limiting their applicability to new image types. Given that labeling images is a primary and common application scenario for interactive segmentation, the reliance on high-quality training data is a significant drawback. (3)Minimal outliers: A good interactive segmentation model should minimize the production of noise and outliers. If the model generates a substantial amount of noise, it becomes tedious and time-consuming to manually remove these inaccuracies, ultimately defeating the purpose of computer-assisted segmentation by bringing extra workloads.

Following those principles, in this paper, we present our method, called the {\it Quasi-Conformal Interactive Segmentation (QIS)}, as a solution to the aforementioned challenges. QIS is a deformable model designed for image segmentation, utilizing the fidelity term of the Mumford-Shah model \cite{mumford1989optimal}. It achieves segmentation by optimizing a quasi-conformal mapping that can effectively warp a template mask to fit a desired region of interest. The core of our approach relies on a quasi-conformal registration-based segmentation, which forms the backbone of our method. User interactions are defined by two actions: a positive click and a negative click, represented by left and right clicks, respectively. These clicks serve to indicate the regions that the user wants to include or exclude from the segmentation process. To incorporate user input seamlessly, we transform the clicks into a click map, which is then integrated into the proposed model using a novel interactive segmentation energy term that we have designed. The interactive segmentation energy function has been carefully analyzed and designed to ensure that the model accurately includes or excludes the regions of interest or disinterest indicated by the click map associated with each click. Detailed explanations of our method and the energy term can be found in Section \ref{sec:method}.

One of the benefits of our method is its utilization of the theoretical foundation of quasi-conformal geometry to generate a bijective mapping. This results in the warped mask maintaining the same topology as the template mask. This property proves advantageous in avoiding outliers and discontinuities in the segmentation results, thereby reducing the need for manual correction significantly. Additionally, the fixed and known topological structure allows for the analysis of optimality for the fidelity term of the Mumford-Shah model. This analysis, as demonstrated in Section \ref{sec:method}, provides theoretical support for the effectiveness of our proposed method in accurately including or excluding regions of interest or disinterest.

To evaluate the performance and advantages of our proposed Quasi-conformal Interactive Segmentation method, we conducted four comprehensive experiments on various image types, including synthetic images, medical images, natural images, and degraded images corrupted by noise. Additionally, we compared our method against alternative approaches. The experimental results unequivocally demonstrate the efficacy of our proposed method.

To summarize, the contributions of this work are as follows:

\smallskip

\begin{itemize}
    \item We propose an interactive segmentation energy term to incorporate user interactions. In particular, we provide an analysis of this model guaranteeing its ability to accurately include or exclude regions of interest or disinterest. This offers a comprehensive analysis and theoretical foundation for our interactive segmentation method.
    
    \item Our interactive segmentation model is based on the principles of quasi-conformal geometry. Leveraging this theory, our model ensures that the resulting segmentation maintains the same topology as the predefined template mask while exhibiting robustness against noise.

    \item An iterative scheme based on the generalized Gauss-Newton framework is presented to numerically solve the proposed interactive segmentation model.
\end{itemize}
\section{Related works}
In this section, we review some related works from three aspects, including computational quasi-conformal geometry, interactive segmentation and deformable model.

\subsection{Computational quasi-conformal geometry}
Computational quasi-conformal geometry is a mathematical tool for analyzing and managing geometric distortions in mappings. Conformal mappings, a subset of quasi-conformal mappings, are crucial in geometry processing, particularly in tasks like texture mapping and surface parameterizations \cite{gu2004genus,gu2003global,levy2002least,lui2014teichmuller}. The Beltrami coefficient quantitatively assesses local geometric distortions and allows precise control over mapping characteristics, aiding in minimizing conformality distortion \cite{choi2016spherical,choi2020free}.

\subsection{Interactive segmentation}
\han{Interactive segmentation is typically formulated as an optimization problem \cite{grady2006random,gulshan2010geodesic,kim2010nonparametric}}. Rother et al. developed GrabCut by solving an optimization problem over graph representation \cite{rother2004grabcut}. Badshah and Chen \cite{rada2014variational} developed a selective segmentation method using active contours. \cite{rada2012new,rada2014variational} proposed the Rada-Chen model, solved with a level-set approach. Spencer and Chen \cite{spencer2015convex,spencer2018parameter} introduced a model with a global minimizer independent of initial user input, later improved by a multigrid algorithm \cite{roberts2019multigrid}. Roberts et al. \cite{roberts2019convex} incorporated edge-weighted geodesic distance, while Ali et al. \cite{ali2021image} added an area-based fitting term for multi-region segmentation. DIOS \cite{xu2016deep} pioneered deep learning in interactive segmentation, embedding clicks into distance maps and using them with the original image. Subsequent works \cite{li2018interactive,liew2019multiseg} addressed ambiguity by predicting multiple potential results. FCANet \cite{lin2020interactive} used the first click for visual attention construction. BRS \cite{jang2019interactive} introduced online optimization for model updates during annotation, which f-BRS \cite{sofiiuk2020f} accelerated by targeting specific layers. CDNet \cite{chen2021conditional} employed self-attention for consistency, and RITM \cite{sofiiuk2022reviving} added the previous mask as input for robustness.

\subsection{Deformable model}
Various research groups have delved into the exploration of deformable models for image segmentation, aiming to derive segmentation results by seeking suitable deformations. An exemplar in this domain is the active contour model \cite{kass1988snakes}, manipulating a set of points discretizing a curve to encapsulate the object's boundary. Cootes et al. \cite{cootes1995active} extend this model to a learnable variant by extracting principal components through principle component analysis (PCA). More recent deformable image segmentation models incorporate a dense spatial deformation map between a template image and a target image, deforming the template mask to match the object's shape in the image.

Chen et al. \cite{chen2021generalized} introduce a dual-front scheme based on asymmetric quadratic metrics, integrating image features and a vector field derived from the evolving contour. Chan et al. \cite{chan2018topology} propose a deformation-based segmentation model using quasi-conformal maps, ensuring topology preservation in the segmentation results. Siu et al. \cite{siu2020image} incorporate the dihedral angle in the deformable model, considering partial convexity and topology constraints. Zhang et al. \cite{zhang2021topology} present a deformable model using the hyperelastic regularization, leading to a topology-preserving segmentation model applicable to 3D volumetric images. Convexity priors are further integrated into the model in \cite{zhang2021topoconv}. In recent years, learning-based models leveraging deformable structures have gained attention, particularly since the introduction of spatial transformer networks \cite{jaderberg2015spatial}. To enhance topology preservation in the final output, Lee et al. \cite{lee2019tetris} apply the Laplacian regularization and Zhang \etal \cite{zhang2022topology} proposed the Relu-Jacobian regularization. Despite these advancements, some of these methods lack a mathematical guarantee of topology preservation or may yield suboptimal results, especially for structures with complex geometries, due to potential over-constraints.
\section{Quasi-conformal segmentation model}

In this section, we briefly review the theory of the quasi-conformal geometry and its applications to image segmentation, which is related to this paper.

\subsection{Mathematical background on quasi-conformal geometry}
Mathematically, an orientation-preserving homeomorphism $f(z):\mathbb{C}\rightarrow\mathbb{C}$ is said to be quasi-conformal if it satisfies the following Beltrami equation \cite{lehto1973quasiconformal}
\begin{equation*}
\frac{\partial f}{\partial \bar{z}} = \mu(z)\frac{\partial f}{\partial z}
\end{equation*}
for some complex-valued Lebesgue measurable function $\mu(z):\mathbb{C}\rightarrow\mathbb{C}$ satisfying $\|\mu(z)\|_{\infty}<1$. Here, $\mu(z)$ is called the Beltrami coefficient \cite{bers1977quasiconformal} and $\|\cdot\|_{\infty}$ represents the infinity norm. Let $z\in\mathbb{C}$ be $x_{1}+\bm{i} x_{2}, x_{1},x_{2}\in\mathbb{R}$ and $f(z) = f_{1}(x_{1},x_{2})+\bm{i} f_{2}(x_{1},x_{2})$. Then by the Wirtinger derivative \cite{remmert1991theory}, $\frac{\partial f}{\partial \bar{z}} = \frac{1}{2}(\frac{\partial f}{\partial x_{1}}+\bm{i}\frac{\partial f}{\partial x_{2}})$ and $\frac{\partial f}{\partial z} = \frac{1}{2}(\frac{\partial f}{\partial x_{1}}-\bm{i}\frac{\partial f}{\partial x_{2}})$, we have
\begin{equation}\label{compute_BC}
\begin{split}
|\mu(z)|^{2} &= \left|\frac{\partial f}{\partial \bar{z}}/\frac{\partial f}{\partial z}
\right |^{2} \\
& = \frac{(\frac{\partial f_{1}}{\partial x_{1}}-\frac{\partial f_{2}}{\partial x_{2}})^{2}+(\frac{\partial f_{2}}{\partial x_{1}}+\frac{\partial f_{1}}{\partial x_{2}})^{2}}{(\frac{\partial f_{1}}{\partial x_{1}}+\frac{\partial f_{2}}{\partial x_{2}})^{2}+(\frac{\partial f_{2}}{\partial x_{1}}-\frac{\partial f_{1}}{\partial x_{2}})^{2}}\\
& = \frac{\|\nabla f\|_{\mathrm{F}}^{2}-2\mathrm{det}\nabla f}{\|\nabla f\|_{\mathrm{F}}^{2}+2\mathrm{det}\nabla f},
\end{split}
\end{equation}
where $\nabla f = \begin{pmatrix}\frac{\partial f_{1}}{\partial x_{1}} & \frac{\partial f_{1}}{\partial x_{2}} \\ \frac{\partial f_{2}}{\partial x_{1}} & \frac{\partial f_{2}}{\partial x_{2}} \end{pmatrix}$ is the Jacobian of $f$, $\mathrm{det}$ is the determinant, and $\|\cdot\|_{\mathrm{F}}$ is the Frobenius norm. From \eqref{compute_BC}, it is easy to derive that $\|\mu(z)\|_{\infty}<1 \Leftrightarrow \mathrm{det}\nabla f>0$, which ensures that a quasi-conformal mapping is one-to-one by the inverse function theorem \cite{lang2012calculus}. In other words, any diffeomorphic deformation must be a quasi-conformal mapping. In addition, we can also see that $\mu(z)=0$ if and only if the Cauchy-Riemann equations, $\frac{\partial f_{1}}{\partial x_{1}}=\frac{\partial f_{2}}{\partial x_{2}}$ and $\frac{\partial f_{2}}{\partial x_{1}}=-\frac{\partial f_{1}}{\partial x_{2}}$, are satisfied. This indicates that the quasi-conformal mapping is the generalization of the conformal mapping. 

The Beltrami coefficient $\mu$ provides us lots of information about the map $f$. In the infinitesimal scale, with respect to its local parameter, a quasi-conformal mapping $f$ can be approximated as follows
\begin{equation}\label{f_expansion}
f(z)\approx f(0)+f_{z}(0)z+f_{\bar{z}}(0)\bar{z} = f(0)+f_{z}(0)(z+\mu(0)\bar{z}).
\end{equation}
From \eqref{f_expansion}, due to the Beltrami coefficient $\mu$ of $f$, we can see that the nonconformal part of $f$ entirely comes from $S(z) = z+\mu(0)\bar{z}$. $S(z)$ is the map that causes $f$ to map a small circle to a small ellipse. According to $\mu(0)$, we can also determine the angles of the directions of maximal magnification and shrinking. Specifically, the angle of maximal magnification is $\arg(\mu(0))/2$ with magnifying factor $1+|\mu(0)|$; the angle of maximal shrinking is the orthogonal angle $(\arg(\mu(0))-\pi)/2$ with shrinking factor $1-|\mu(0)|$. So the Beltrami coefficient $\mu(z)$ can represent the distortion and measure how far away the quasi-conformal map at each point is deviated from a conformal map (Figure \ref{fig:conformality}).

\begin{figure}[htbp!]
    \centering
    \includegraphics[width=0.4\textwidth]{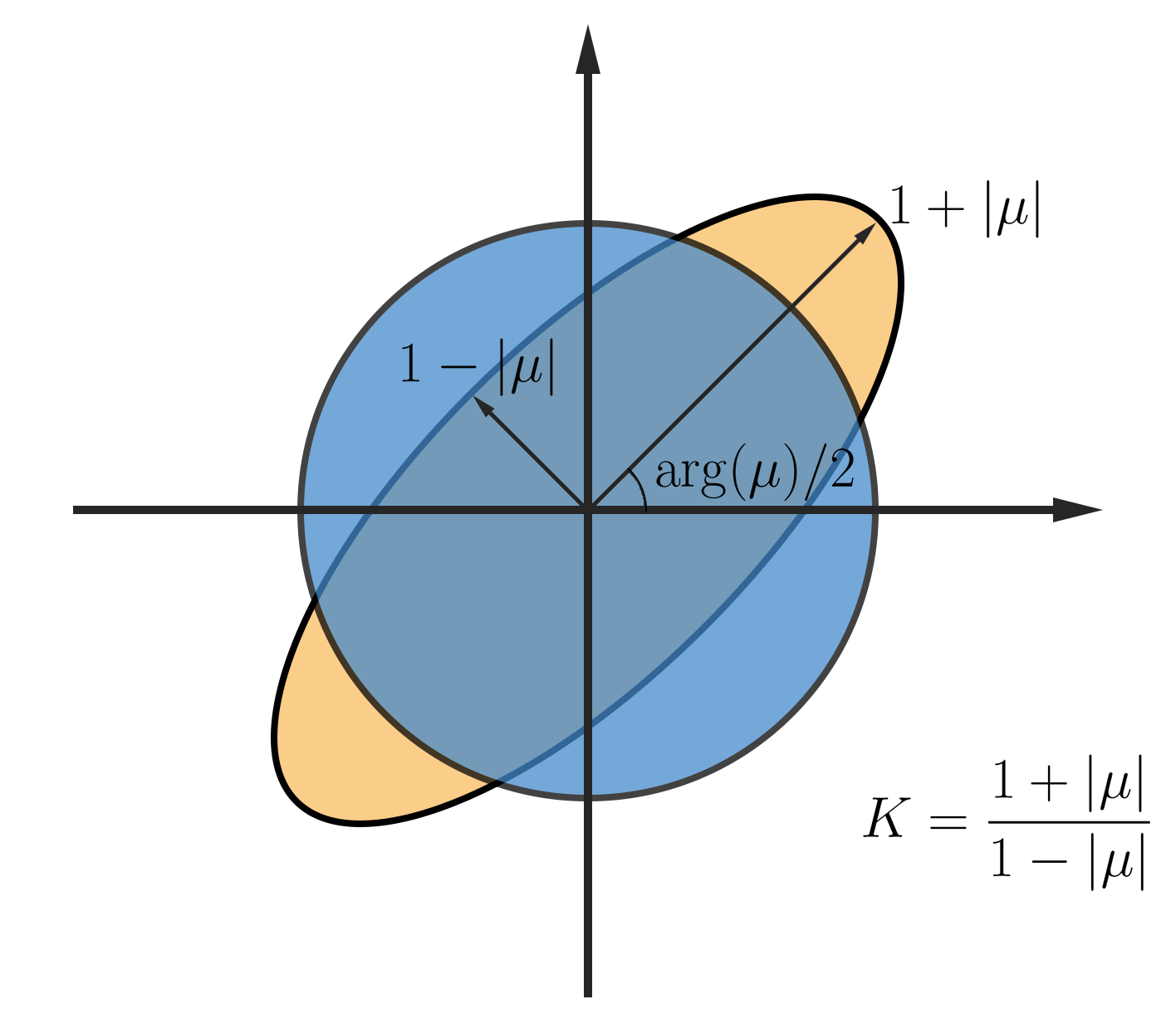}
    \caption{An illustration of the relationship between the conformality distortion and the Beltrami coefficient of a quasi-conformal mapping.}
    \label{fig:conformality}
\end{figure}

Further, denote $\sigma_{1}\geq\sigma_{2}>0$ as the singular values of $\nabla f$. Then by \eqref{compute_BC} again, we have $|\mu(z)|=\frac{\sigma_{1}-\sigma_{2}}{\sigma_{1}+\sigma_{2}}$. We define by $K(z)$ the dilatation 
\begin{equation*}
K(z) = \frac{1+|\mu(z)|}{1-|\mu(z)|}
\end{equation*}
to express the ratio of the largest singular value of $\nabla f$ divided by the smallest one. So $K(z)$ can also measure how far away the quasi-conformal map at each point is deviated from a conformal map. Hence, based on the above discussion, we can derive a diffeomorphic deformation with bounded geometric distortion by restricting $\mu(z)$ or $K(z)$.

For more details about the quasi-conformal theory, please refer to \cite{gardiner2000quasiconformal,lehto1973quasiconformal}.

\subsection{Quasi-conformal registration-based segmentation model}\label{sec:qcsm}
The registration-based segmentation model is to extract the interested region by employing the registration framework. For doing the registration, we require template and reference. We can prescribe a prior image as the template and take the target image as the reference, where the structure of the prior image is related to the target image. 
Ideally, after doing the registration, the boundary of the object in the target image can be linked to the boundary of the object in the prior image by the resulting transformation. 
In addition, if the resulting transformation is one-to-one, a topology-preserving segmentation is guaranteed, which means that the topological structure of the segmentation result is the same with the topological structure of the prior image. 

Denote $I(\bm{x}):\Omega\subset\mathbb{R}^{2}\rightarrow \mathbb{R}$ as the target image. For simplicity, we consider a two-phase segmentation, namely, segmenting the image domain into the foreground and background. In this way, we can define the prior image as
\begin{equation*}
J(c_{1},c_{2},\bm{x}) = c_{1}\mathcal{X}_{D}(\bm{x})+c_{2}\mathcal{X}_{D^{c}}(\bm{x}), 
\end{equation*}
where $D\subset \Omega$ is a mask, $D^{c}$ is the complement of $D$, and 
\begin{equation*}
\mathcal{X}_{D}(\bm{x}) = 
\begin{cases}
1, \quad \bm{x}\in D,\\
0, \quad \mathrm{otherwise},
\end{cases}
\end{equation*}
is the indicator function of $D$. Then, we can build the variational framework of the registration-based segmentation model
\begin{equation*}
\min_{c_{1},c_{2},\varphi(\bm{x})} \frac{1}{2}\int_{\Omega}(I(\bm{x})-J(c_{1},c_{2},\varphi(\bm{x})))^{2}\mathrm{d}\bm{x}+\alpha\mathcal{R}(\varphi(\bm{x})),
\end{equation*}
where $\varphi(\bm{x}):\mathbb{R}^{2}\rightarrow\mathbb{R}^{2}$ is the transformation, $\mathcal{R}(\varphi(\bm{x}))$ is the regularization term to avoid outliers, and $\alpha>0$ is a nonnegative parameter to balance the weight between the fitting term and regularization term. 

There exist many choices for the regularizers, such as the elastic regularizer \cite{broit1981optimal}, the curvature regularizer \cite{chumchob2011fourth,fischer2003curvature,fischer2004unified,ibrahim2015novel} and the fractional-order regularizer \cite{zhang2015variational}. However, the most commonly used regularizers can't ensure a one-to-one transformation because they do not involve the information of the Jacobian determinant of the transformation. Hence, in this paper, to obtain a topology-preserving segmentation model, we mainly consider the following quasi-conformal registration-based segmentation model \cite{chan2018topology}
\begin{equation}\label{eq:qcsm1}
\min_{c_{1},c_{2},\varphi} \frac{1}{2}\int_{\Omega}(I-J(c_{1},c_{2},\varphi))^{2}\mathrm{d}\bm{x}+\alpha_{1}\int_{\Omega}|\Delta\varphi|^{2}\mathrm{d}\bm{x}+\alpha_{2}\int_{\Omega}\phi(|\mu(\varphi)|^{2})\mathrm{d}\bm{x},
\end{equation}
where $\Delta$ is the Laplace operator, $|\mu(\varphi)|^{2}$ is computed by \eqref{compute_BC}, and $\phi(v)$ is a scaling function defined for positive $v$ which will be chosen in Section \ref{sec:numerical}. This scaling function should scale up the value of $v$ when it approaches $1$ and scale down if it is near zero. Through this scaling, the value of $v$ should be regularized to be smaller than $1$ as the energy will be dramatically increased when $v$ goes toward $1$. $\alpha_{1}$ and $\alpha_{2}$ are nonnegative parameters weighting the two regularization terms. We can see that the second term is to control the smoothness and the third term is to restrict the distortion.

\section{Method}\label{sec:method}
In this section, we detail our proposed Quasi-conformal Interactive Segmentation (QIS) model. As an interactive model, our model solves a quasi-conformal registration-based segmentation model \eqref{eq:qcsm1} in each step. By employing the quasi-conformal approach, the mapping for the deformable segmentation model is guaranteed to be bijective, thus preserving the topology. This control over topology helps avoid noise and outliers in the results. Additionally, we provide theoretical analyses that support our model's ability to efficiently incorporate user input. These analyses demonstrate that with each new click, the local region around the click will be accurately included or excluded, ensuring an effective and precise segmentation.

\subsection{Overall model}\label{sec:overallmodel}

\begin{figure}
    \centering
    \includegraphics[width=\textwidth]{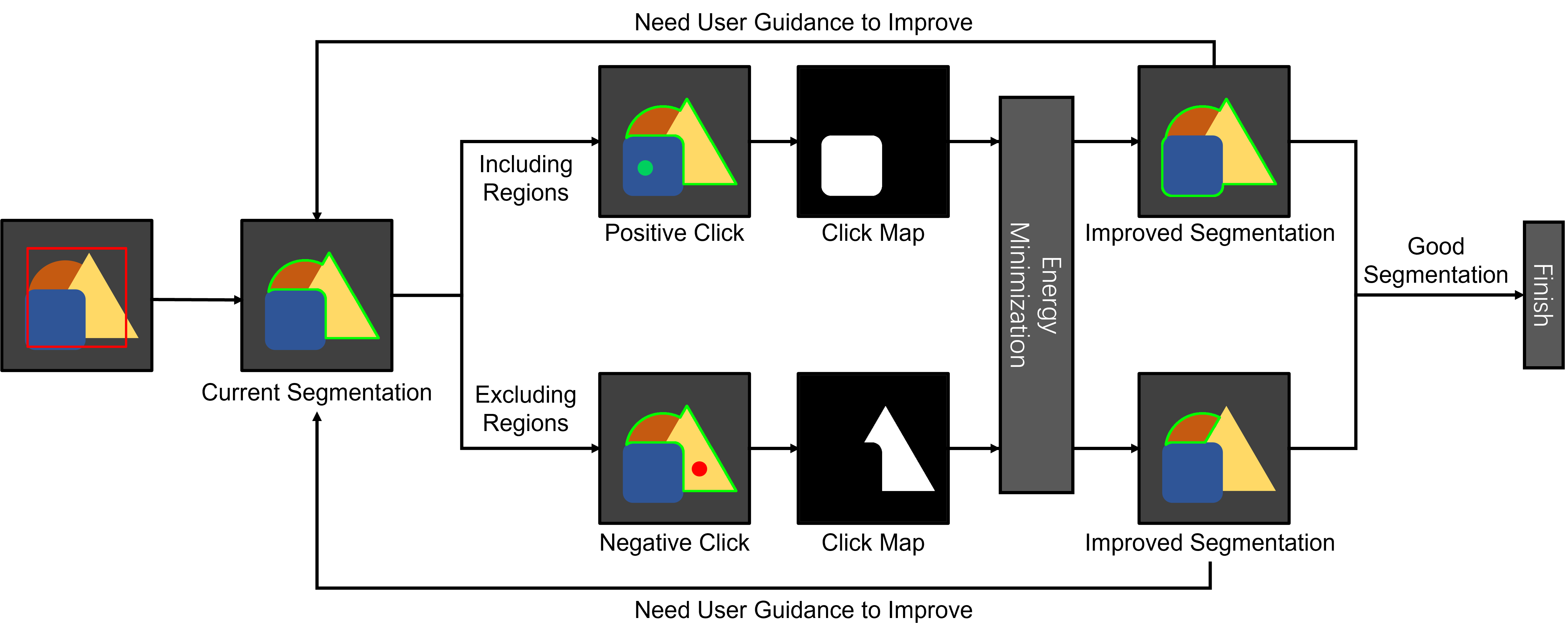}
    \caption{A flow chart illustrating the process of the quasi-conformal interactive segmentation (QIS) model.}
    \label{fig:pipeline}
\end{figure}

In this subsection, we first provide an overview of the interactive segmentation procedure. Figure \ref{fig:pipeline} visually illustrates the process. Initially, the quasi-conformal registration-based segmentation model, described in \eqref{eq:qcsm1}, is applied to obtain an initial segmentation. This initial segmentation serves as the current segmentation mask, which may not be entirely accurate and thus requires further adjustments. To guide these adjustments, the user can indicate regions to be included (using positive clicks) and regions to be excluded (using negative clicks). These clicks are then converted into a click map $M_{\mathrm{clicks}}$, which is subsequently incorporated into the quasi-conformal segmentation model. By minimizing the interactive segmentation energy, a new quasi-conformal map is generated, resulting in an improved segmentation mask. If the improved segmentation is still not satisfactory, the process is repeated until a satisfactory segmentation result is achieved. This iterative procedure, which incorporates the user's input, is referred to as the {\it interactive segmentation process}.

Interactive segmentation involves iteratively updating the segmented region until it encompasses the entire region of interest. As a result, the overall model can be divided into multiple steps, with each step focusing on introducing an appropriate click map $M_{\mathrm{clicks}}$ and segmenting a new image. Therefore, the overall process for our interactive segmentation model can be outlined as follows.

\textbf{Step 0}. Given a initial mask $D$ and parameters $\alpha_1$ and $\alpha_2$, set $I^0 = I$ and compute $\varphi^{0}$ by solving the following model:
\begin{equation*}
\min_{c_{1},c_{2},\varphi}  \frac{1}{2}\int_{\Omega}(I^0 - c_1 \mathcal{X}_{D} \circ \varphi - c_2 \mathcal{X}_{D^c} \circ \varphi)^{2}\mathrm{d}\bm{x}+\alpha_{1}\int_{\Omega}|\Delta\varphi|^{2}\mathrm{d}\bm{x}+\alpha_{2}\int_{\Omega}\phi(|\mu(\varphi)|^{2})\mathrm{d}\bm{x}.
\end{equation*}

\textbf{Step 1}. Give the positive/negative click, compute the click map $M^{1}_{\mathrm{clicks}}$ by \eqref{eq:compute_clickmap}, choose a suitable weight $r^{1}$ by \eqref{eq:choose_positveclick} or \eqref{eq:choose_negativeclick} and set $I^{1} = I^0+r^{1}M^{1}_{\mathrm{clicks}}$. Then given $\varphi^{0}$ as the initial guess, compute $\varphi^{1}$ by solving the following model:
\begin{equation*}
\min_{c_{1},c_{2},\varphi}  \frac{1}{2}\int_{\Omega}(I^{1} - c_1 \mathcal{X}_{D} \circ \varphi - c_2 \mathcal{X}_{D^c} \circ \varphi)^{2}\mathrm{d}\bm{x}+\alpha_{1}\int_{\Omega}|\Delta\varphi|^{2}\mathrm{d}\bm{x}+\alpha_{2}\int_{\Omega}\phi(|\mu(\varphi)|^{2})\mathrm{d}\bm{x}.
\end{equation*}

\begin{center}
$\vdots$
\end{center}

\textbf{Step n}. Give the positive/negative click, compute the click map $M^{n}_{\mathrm{clicks}}$ by \eqref{eq:compute_clickmap}, choose a suitable weight $r^{n}$ by \eqref{eq:choose_positveclick} or \eqref{eq:choose_negativeclick} and set $I^{n} = I^{n-1}+r^{n}M^{n}_{\mathrm{clicks}}$. Then given $\varphi^{n-1}$ as the initial guess, compute $\varphi^{n}$ by solving the following model:
\begin{equation}\label{eq:step-n}
\min_{c_{1},c_{2},\varphi}  \frac{1}{2}\int_{\Omega}(I^{n} - c_1 \mathcal{X}_{D} \circ \varphi - c_2 \mathcal{X}_{D^c} \circ \varphi)^{2}\mathrm{d}\bm{x}+\alpha_{1}\int_{\Omega}|\Delta\varphi|^{2}\mathrm{d}\bm{x}+\alpha_{2}\int_{\Omega}\phi(|\mu(\varphi)|^{2})\mathrm{d}\bm{x}.
\end{equation} 

Note that in each step of the overall model, we solve a quasi-conformal registration-based segmentation model \eqref{eq:qcsm1}. This model can be regarded as the backbone of the whole method. The topology-preserving of the backbone provides great efficiency and helps to segment the images affected by noise or degraded by overexposure or underexposure \cite{zhang2021topology,zhang2021topoconv}. Such advantages brought by the topology-preserving property are further illustrated by the improvement of the segmentation accuracy given in Section \ref{Experiments}. The segmentation results are iteratively refined based on user input, with each click transforming into a click map. This click map, along with a carefully selected parameter 
$r$ (discussed in Section \ref{sec:choiceofr}), is used to update $I^{n-1}$ to $I^n$ as outlined earlier. The segmentation mask at step $n$ is then derived by solving the sub-problem specified in equation \eqref{eq:step-n}. This iterative process continues until the segmentation results are deemed satisfactory.

In the following subsections, we will discuss the definitions of the positive click, negative click, and click the map in detail. Additionally, we will provide a comprehensive explanation of the interactive segmentation fidelity term, which will reveal how to choose the weight $r$.

\subsection{Click map}\label{sec:clickmap}

In this subsection, we introduce the positive click, negative click, and click map, which is essential for constructing our proposed interactive segmentation model.

Mis-segmentation may result in the exclusion of certain regions from the segmentation mask. To address this, users actively participate in the process by interactively delineating points within these excluded regions. These points are referred to as positive clicks. We denote the collection of positive clicks by $C_p$. Similarly, there may be instances where regions that do not belong to the object of interest are mistakenly included in the segmentation mask. To address this, users assist by delineating points within these regions. These points are referred to as negative clicks, as they indicate areas that should be excluded from the final segmentation. The collection of negative clicks is denoted as $C_n$.

The concepts of positive clicks and negative clicks are introduced to minimize the level of user intervention required in identifying specific regions for inclusion or exclusion. From these clicks, a click map can be generated, which effectively captures the local homogeneous regions that are grown around the user's clicks. These local homogeneous regions approximate the regions to be included for positive clicks and the regions to be excluded for negative clicks. This approach ultimately enhances the user experience by streamlining the segmentation process and significantly reducing the manual intervention required.

\begin{figure}
    \centering
    \includegraphics[width=0.8\textwidth]{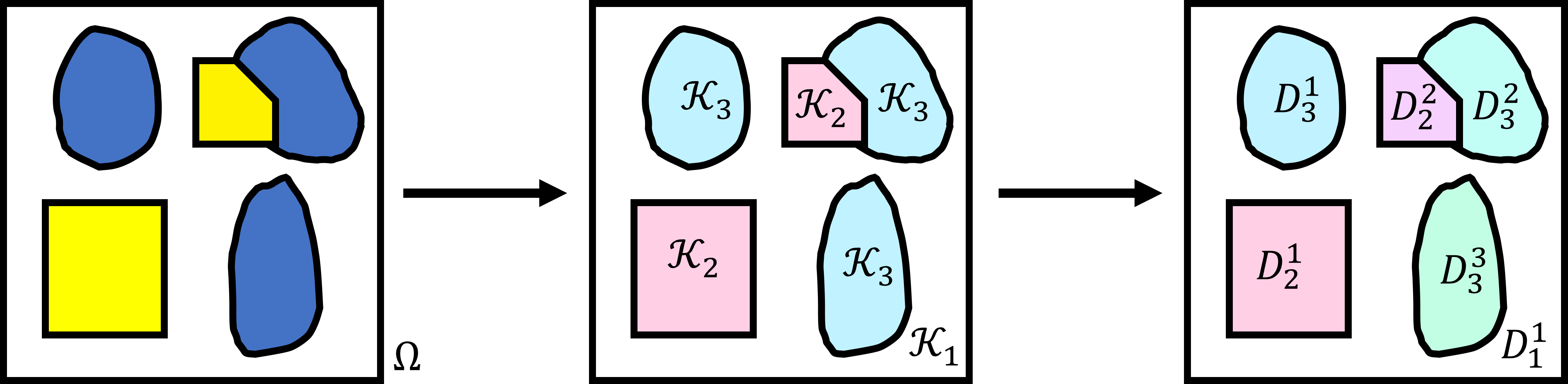}
    \caption{An example demonstrating the division of isolated regions in $\mathcal{K}_i$ into multiple $D^i_j$ for $k=3$. In this example, the image domain is classified into three distinct regions based on pixel values: (1) background white, (2) blue, and (3) yellow. Subsequently, the disconnected regions within each class $\mathcal{K}_i$ are further divided into individual $D^i_j$'s.}
    \label{fig:kmeanclickmap}
\end{figure}

In this work, we propose a methodology to generate the click map from user clicks based on $K$-means clustering. Suppose a collection $\mathcal{C}_*$ of positive clicks or negative clicks is given. In other words, $\mathcal{C}_* = \mathcal{C}_p$ or $\mathcal{C}_n$. Illustrated by Figure~\ref{fig:kmeanclickmap}. the image domain $\Omega$ can be decomposed into $K$ clusters based on the image intensity $I:\Omega\to \mathbb{R}$, using the $K$-means clustering. Note that each cluster $\mathcal{K}_i$, may consist of multiple isolated components. We denote the isolated components by $D^i_j$. Hence, $\mathcal{K}_i = \bigcup_{j=1}^{N_i} D^i_j$, where $N_i$ indicates the number of isolated regions within cluster $i$. For each $\bm{x}_{\text{click}}\in\mathcal{C}_*$, we define $\delta(\bm{x}_{click})$ as follows:
\begin{equation*}\label{clusteringbasedmap}
\begin{aligned}
&\delta(\bm{x}_{click}) = {D^i_j} \text{ for } (i,j) \text{ if } \bm{x}_{click} \in D^i_j.
\end{aligned}
\end{equation*} 
A local homogeneous region surrounding clicks can be obtained by taking the union, which is given by $\mathcal{H} = \bigcup_{\bm{x}_{click} \in\mathcal{C}_*} \delta(\bm{x}_{click})$. Then a click map $M_{\text{click}}$, which is a binary map capturing $\mathcal{H}$, can be defined as follows.

\begin{definition}[Click map]
    The click map $M_{\text{clicks}}$ for clicks $\mathcal{C}_*$ is defined as:
    \begin{equation}\label{eq:compute_clickmap}
        M_{\text{clicks}}(\bm{x}) = 
        \begin{cases}
        1, & \text{if } \bm{x} \in \mathcal{H},\\
        0, & \text{otherwise}.
        \end{cases}
    \end{equation}
\end{definition}

Note that the definition provided above accommodates not only single clicks but also sets of clicks as input. This formulation enables the convenient use of drawing continuous lines instead of individual clicks, which can be directly converted into a collection of densely sampled points along the line. This approach enhances the efficiency of providing guidance. 

\subsection{Fidelity term for interactive segmentation}\label{FidelityTermforInteractiveSegmentation}
\label{sec:InteractiveSegmentationModel}

The click map provides additional information to refine the segmentation prediction.  In this subsection, we will analyze the properties of the fidelity term of the proposed interactive segmentation, which prompts us how to choose the suitable weight $r$. 

\subsubsection{Revisit on the fidelity term of \eqref{eq:qcsm1}}
In this part, we thoroughly analyze the fidelity term of \eqref{eq:qcsm1}. This analysis will facilitate our development of the interactive segmentation model.

For the fitting term of \eqref{eq:qcsm1}, we set 
\begin{equation}
    E(c_{1},c_{2},\varphi) := \int_{\Omega}(I - J(c_{1},c_{2},\varphi))^{2}\mathrm{d}\bm{x} = \int_{\Omega}(I - c_1 \mathcal{X}_{D} \circ \varphi - c_2 \mathcal{X}_{D^c} \circ \varphi)^{2}\mathrm{d}\bm{x}.
    \label{eq:segmentationD}
\end{equation}

Since $\varphi$ is a one-to-one mapping, the deformed region $D$ can be set as a region $G$ and $\mathcal{X}_G = \mathcal{X}_D \circ \varphi$. Hence, \eqref{eq:segmentationD} can be rewritten as
\begin{equation}
    E(c_{1},c_{2},G) = \int_{\Omega}(I - c_1 \mathcal{X}_{G} - c_2 \mathcal{X}_{G^c})^{2}\mathrm{d}\bm{x}.
    \label{eq:segmentationG}    
\end{equation}

A segmentation mask $G$ can divide an image domain into a foreground and a background. However, since the segmentation may be imperfect, there could be segmented regions that are not part of the true region of interest, which we call false positive. Those regions that are classified to be the background but are the actual region of interest are called false negative. To simplify our analysis, we can assume two cases: (1) the case with only false positive and (2) the case with only false negative. Thus, we can assume the segmentation result to have three domains, true positive (the correctly segmented foreground), true negative (the correctly segmented background), and false positive/false negative.

\begin{figure}[htbp]
    \centering
    \includegraphics[width=0.25\textwidth]{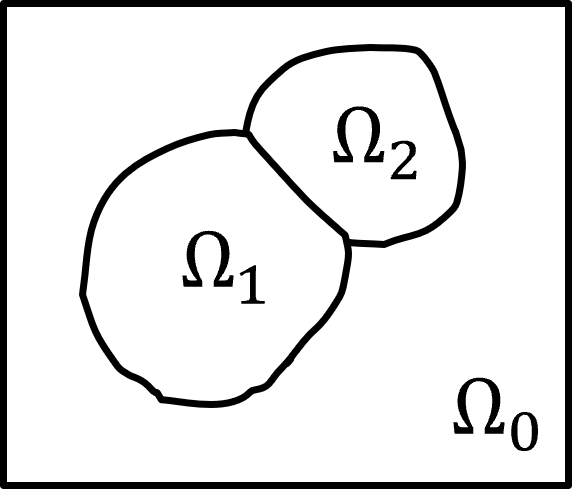}
    \caption{Notations for a three-value image with three domains: $\Omega_{0}$ (background), $\Omega_{1}$ (foreground - wanted), and $\Omega_{2}$ (foreground - unwanted).}
    \label{fig:simpleimage}
\end{figure}

Following these considerations, as Figure \ref{fig:simpleimage}, we set a three-value image $I$ as follows
\begin{equation}
I =
\begin{cases}
p_0, &\bm{x} \in \Omega_0,\\
p_1, &\bm{x} \in \Omega_1,\\
p_2, &\bm{x} \in \Omega_2,\\
\end{cases}    
\label{eq:imagedef}
\end{equation}
where the domain $\Omega_0$ is the background region. Denote the area of some domain $\Omega$ as $||\Omega||$. Then the following theorem indicates that minimizing the segmentation energy \eqref{eq:segmentationG} can only lead to three cases: (1)$E(c_{1},c_{2}, G=\Omega_1)$; (2)$E(c_{1},c_{2}, G=\Omega_2)$; (3)$E(c_{1},c_{2}, G=\Omega_1 \cup \Omega_1)$.

\begin{theorem}\label{theorem_three_value}    
The segmentation energy \eqref{eq:segmentationG} for a three-value image $I$ only has three local minimizes. They are
\begin{enumerate}
    \item $G^* = \Omega_1$;
    \item $G^* = \Omega_2$;
    \item $G^* = \Omega_1 \cup \Omega_2$.
\end{enumerate}
This means that no subset of $\Omega_i$ can be included.
\end{theorem}
\begin{proof}
See Appendix \ref{proof1}.
\end{proof}

Next, we compute the corresponding energy value with respect to these different cases. When $G=\Omega_1$, we have $c_1 = p_1$ and $c_2 = \frac{p_2 ||\Omega_2|| + p_0 ||\Omega_0||}{||\Omega_2|| + ||\Omega_0||}$.
Then, we can get 
\begin{equation}
\begin{aligned}
    E(c_{1},c_{2},G=\Omega_1)
    &= \int_\Omega (I - c_1 \mathcal{X}_{\Omega_1} - c_2 \mathcal{X}_{\Omega_2 \cup \Omega_0})^2 \mathrm{d}\bm{x}\\
    &= (p_0 - c_2)^2 ||\Omega_0|| + (p_2 - c_2)^2 ||\Omega_2||\\    
    &= (p_0 - p_2)^2 \frac{||\Omega_2|| ||\Omega_0||}{||\Omega_2|| + ||\Omega_0||}.
\end{aligned}
\label{eq:E1}
\end{equation}
Similarly, when $G=\Omega_2$, we have $c_1 = p_2$, $c_2 = \frac{p_1 ||\Omega_1|| + p_0 ||\Omega_0||}{||\Omega_1|| + ||\Omega_0||}$, and
\begin{equation}
    E(c_{1},c_{2},G=\Omega_2) = (p_0 - p_1)^2 \frac{||\Omega_1|| ||\Omega_0||}{||\Omega_1|| + ||\Omega_0||}.
\label{eq:E2}
\end{equation}
And for $G = \Omega_1 \cup \Omega_2$, we get $c_1 = \frac{p_1 ||\Omega_1|| + p_2 ||\Omega_2||}{||\Omega_1|| + ||\Omega_2||}$, $c_2 = p_0$, and
\begin{equation}
    E(c_{1},c_{2},G=\Omega_1 \cup \Omega_2) = (p_1 - p_2)^2 \frac{||\Omega_1|| ||\Omega_2||}{||\Omega_1|| + ||\Omega_2||}.
\label{eq:E12}
\end{equation}
Hence, for a clear three-value image, we can optimize the energy functional \eqref{eq:segmentationG} and obtain $G^*$ by comparing the energy values of the above three cases: 
\begin{itemize}
\item $G^* = \Omega_1$, if
\begin{equation}
\begin{cases}
    (p_0-p_2)^2 \frac{||\Omega_2||}{||\Omega_2|| + ||\Omega_0||} \leq (p_0-p_1)^2\frac{||\Omega_1||}{||\Omega_1|| + ||\Omega_0||}, \\
    (p_0-p_2)^2 \frac{||\Omega_0||}{||\Omega_2|| + ||\Omega_0||} \leq (p_1-p_2)^2\frac{||\Omega_1||}{||\Omega_1|| + ||\Omega_2||}.
\end{cases}
\label{eq:omega1}
\end{equation}
\item $G^* = \Omega_2$, if
\begin{equation}
\begin{cases}
    (p_0-p_1)^2 \frac{||\Omega_1||}{||\Omega_1|| + ||\Omega_0||} \leq (p_0-p_2)^2\frac{||\Omega_2||}{||\Omega_2|| + ||\Omega_0||}, \\
    (p_0-p_1)^2 \frac{||\Omega_0||}{||\Omega_1|| + ||\Omega_0||} \leq (p_1-p_2)^2\frac{||\Omega_2||}{||\Omega_1|| + ||\Omega_2||}.
\end{cases}
\label{eq:omega2}
\end{equation}
\item $G^* = \Omega_1 \cup \Omega_2$, if
\begin{equation}
\begin{cases}
    (p_1-p_2)^2 \frac{||\Omega_1||}{||\Omega_1|| + ||\Omega_2||} \leq (p_0-p_2)^2\frac{||\Omega_0||}{||\Omega_2|| + ||\Omega_0||}, \\
    (p_1-p_2)^2 \frac{||\Omega_2||}{||\Omega_1|| + ||\Omega_2||} \leq (p_0-p_1)^2\frac{||\Omega_0||}{||\Omega_1|| + ||\Omega_0||}.
\end{cases}
\label{eq:omega12}
\end{equation}
\end{itemize}

\begin{theorem}\label{theorem_comparisioon}
The minimizer $G^*$ of the segmentation energy \eqref{eq:segmentationG} for a clear three-value image $I$ is one of the following three:
\begin{enumerate}
    \item when \eqref{eq:omega1} is satisfied, $G^* = \Omega_1$;
    \item when \eqref{eq:omega2} is satisfied, $G^* = \Omega_2$;
    \item when \eqref{eq:omega12} is satisfied, $G^* = \Omega_1 \cup \Omega_2$.
\end{enumerate}
\end{theorem}

\subsubsection{User's Clicks and Parameter Selection}
\label{sec:choiceofr}
Given a click map $M_{\mathrm{clicks}}: \Omega \longrightarrow \mathbb{R}$, our proposed interactive segmentation fidelity energy is defined as
\begin{equation}
    E_{int}(c_{1},c_{2},G; M_{\mathrm{clicks}}) = \int_\Omega (I + rM_{\mathrm{clicks}} - c_1 \mathcal{X}_{G} - c_2 \mathcal{X}_{G^c})^2 \mathrm{d}\bm{x}.
    \label{eq:interactive}
\end{equation}

With a carefully selected value of $r$ for the provided click map $M_{\mathrm{clicks}}$, the new energy functional \eqref{eq:interactive} is designed to facilitate the inclusion or exclusion of specific regions within the initial segmentation. In the following, we will examine two scenarios. The first scenario involves including the false negative region, which should have been part of the segmentation but was erroneously omitted from the initial segmentation mask. The second scenario involves excluding the false positive region, which is not relevant to the region of interest but was mistakenly included in the initial segmentation mask.

\paragraph{Case 1: Positive click.}
Mathematically, suppose the region of interest is represented by the union of two regions, $\Omega_1 \cup \Omega_2$. However, the initial segmentation only identifies $\Omega_1$, omitting the presence of $\Omega_2$. To rectify this, a positive click is required to indicate that $\Omega_2$ should also be included as part of the region of interest. Therefore, by considering the click map $M_{\mathrm{clicks}}=\mathcal{X}_{\Omega_2}$ and leveraging the interactive segmentation fidelity energy, the parameter $r$ in \eqref{eq:interactive} needs to be carefully selected to ensure that the updated model can generate an optimal segmentation $G^{**} = \Omega_1 \cup \Omega_2$. The following theorem describes how $r$ can be chosen.

\begin{theorem}\label{theorem_positive_click}
Given a three-value image $I$ defined as \eqref{eq:imagedef}, with the initial segmentation $G^* = \Omega_1$ derived by \eqref{eq:segmentationG} and the click map $M_{\mathrm{clicks}}=\mathcal{X}_{\Omega_2}$, we can obtain the desired segmentation $G^{**} = \Omega_1 \cup \Omega_2$ from the interactive segmentation fidelity energy \eqref{eq:interactive} if $r$ satisfies 
\begin{equation}\label{eq:choose_positveclick}
\begin{cases}
    \frac{(p_1-p_2)B - (p_0-p_1)}{B} \leq r \leq \frac{(p_1-p_2)A + (p_0-p_2)}{A+1} \ \text{when $p_0 > p_1$},\\
    \frac{(p_1-p_2)A + (p_0-p_2)}{A+1} \leq r \leq \frac{(p_1-p_2)B - (p_0-p_1)}{B} \ \text{when $p_0 < p_1$},
\end{cases}
\end{equation}
where $A=\sqrt{\frac{||\Omega_1||(||\Omega_2|| + ||\Omega_0||)}{||\Omega_0||(||\Omega_1|| + ||\Omega_2||)}}$ and
$B=\sqrt{\frac{||\Omega_2||(||\Omega_1|| + ||\Omega_0||)}{||\Omega_0||(||\Omega_1|| + ||\Omega_2||)}}$. In practice, $r$ can be chosen as the middle point of the range, i.e. $r= \left(\frac{(p_1-p_2)A + (p_0-p_2)}{A+1}+\frac{(p_1-p_2)B - (p_0-p_1)}{B}\right)/2$.
\end{theorem}
\begin{proof}
See Appendix \ref{proof2}.
\end{proof}

\paragraph{Case 2: Negative click.}
When the initial segmentation generated by \eqref{eq:segmentationG} includes an additional region that is not part of the desired segmentation, a negative click is required to remove it. Assume that the region of interest is denoted as $\Omega_1$, but the initial segmentation erroneously includes an extra region $\Omega_2$. By considering the click map $M_{\mathrm{clicks}}=\mathcal{X}_{\Omega_2}$, we need to carefully choose the parameter $r$ in \eqref{eq:interactive} to ensure that the updated model can produce an optimal segmentation $G^{**} = \Omega_1$. The following theorem provides details on how the value of $r$ can be selected.

\begin{theorem}\label{theorem_negative_click}
Given a three-value image $I$ defined as \eqref{eq:imagedef}, with the initial segmentation $G^* = \Omega_1 \cup \Omega_2$ derived from \eqref{eq:segmentationG} and the click map $M_{\mathrm{clicks}}=\mathcal{X}_{\Omega_2}$, we can obtain the desired segmentation $G^{**} = \Omega_1$ from the interactive segmentation fidelity energy \eqref{eq:interactive} if $r$ satisfies 
\begin{equation}\label{eq:choose_negativeclick}
\begin{cases}
    \frac{(p_0-p_2)C + (p_1-p_2)}{C+1} \leq r \leq \frac{(p_0-p_2)E + (p_0-p_1)}{E} \ \text{when $p_0 > p_1$},\\
    \frac{(p_0-p_2)E + (p_0-p_1)}{E} \leq r \leq \frac{(p_0-p_2)C + (p_1-p_2)}{C+1} \ \text{when $p_0 < p_1$},
\end{cases}
\end{equation}
where $C=\sqrt{\frac{||\Omega_0||(||\Omega_2|| + ||\Omega_1||)}{||\Omega_1||(||\Omega_0|| + ||\Omega_2||)}}$ and
$E=\sqrt{\frac{||\Omega_2||(||\Omega_0|| + ||\Omega_1||)}{||\Omega_1||(||\Omega_0|| + ||\Omega_2||)}}$.
In practice, $r$ can be chosen as the middle point of the range, i.e. $r= \left(\frac{(p_0-p_2)C + (p_1-p_2)}{C+1}+\frac{(p_0-p_2)E + (p_0-p_1)}{E}\right)/2$
\end{theorem}
\begin{proof}
See Appendix \ref{proof3}.
\end{proof}

Our method is an interactive model that involves human input for each step, which leads to some factors in the whole procedure being undetermined. However, in each step where the clicks are given to be negative or positive, the click map can then be obtained, and so is the parameter $r$. With those components fixed, in each step $n$ of the proposed interactive segmentation model, the following theorem can guarantee the existence of the minimizer.

\begin{theorem}[Existence of minimizer]
For each step $n$ of the proposed interactive segmentation model, suppose $\Omega$ is bounded and simply connected, $I^{n-1}$ is a 
function from $\Omega\subset \mathbb{R}^{n}\rightarrow \mathbb{R}$, $r^{n}$ is chosen by \eqref{eq:choose_positveclick} or \eqref{eq:choose_negativeclick}, $M^{n}_{\mathrm{clicks}}$ is computed by \eqref{eq:compute_clickmap}  and $\alpha_{i}>0, i=1, 2$. Let
\begin{equation*}
\mathcal{A}:=\{\varphi\in \mathcal{C}^{2}(\Omega): \|\varphi\|_{\infty}\leq d_1, \|\nabla\varphi\|_{\infty}\leq d_{2}, \|\nabla\varphi^{2}\|_{\infty}\leq d_{3},\det\nabla\varphi\geq\epsilon\}
\end{equation*} 
for some $d_{i}>0, i=1, 2, 3$ and a small $\epsilon>0$. Then each step $n$ of the proposed interactive segmentation model admits a minimizer in $\mathcal{A}$. In fact, $\mathcal{A}$ is compact.
\end{theorem}
\begin{proof}
The proof is similar with Theorem 4.1 in \cite{siu2020image} and \cite{zhang2022unifying}.
\end{proof}

\section{Numerical implementation}\label{sec:numerical}
In this section, we give the details of how to solve the proposed interactive model in Section \ref{sec:overallmodel}. 

The main components in each step $n$ of the proposed model are to compute the click map $M^{n}_{\mathrm{clicks}}$, choose the weight $r^{n}$ and solve the following variational problem 
\begin{equation}\label{eq:eachstep_mainproblem}
\min_{c_{1},c_{2},\varphi}  \frac{1}{2}\int_{\Omega}(I^{n} - c_1 \mathcal{X}_{D} \circ \varphi - c_2 \mathcal{X}_{D^c} \circ \varphi)^{2}\mathrm{d}\bm{x}+\alpha_{1}\int_{\Omega}|\Delta\varphi|^{2}\mathrm{d}\bm{x}+\alpha_{2}\int_{\Omega}\phi(|\mu(\varphi)|^{2})\mathrm{d}\bm{x},
\end{equation}
where $I^{n} = I^{n-1}+r^{n}M^{n}_{\mathrm{clicks}}$.
In Section \ref{sec:clickmap} and \ref{FidelityTermforInteractiveSegmentation}, we have provided the approach to calculate the click map $M^{n}_{\mathrm{clicks}}$ and the corresponding $r^{n}$, respectively. The scaling function that used to better penalize large $v = |\mu(\varphi)|$ is chosen as $\phi(v) = \frac{1}{(v-1)^{2}}$ \cite{zhang2018novel}. Next, we mainly investigate the numerical solver for \eqref{eq:eachstep_mainproblem}.

We see that variables $\bm{c} = (c_{1},c_{2})$ and $\varphi$ are coupled in the problem \eqref{eq:eachstep_mainproblem}, then we employ the alternating direction method, namely first fix $\varphi$ to solve $\bm{c}$ then fix $\bm{c}$ to solve $\varphi$. The $k$-th iterative scheme is listed as follows
\begin{equation}\label{iterativescheme}
\left\{
\begin{aligned}
&(c^{k}_{1},c^{k}_{2}) = \mathrm{argmin}_{c_{1},c_{2}} \int_{\Omega}(I^n - c_1 \mathcal{X}_{D} \circ \varphi^{k-1} - c_2 \mathcal{X}_{D^c} \circ \varphi^{k-1})^{2}\mathrm{d}\bm{x}, \\
& \qquad
\begin{aligned}
\varphi^{k} =  \mathrm{argmin}_{\varphi}  & \frac{1}{2}\int_{\Omega}(I^n - c^{k}_1 \mathcal{X}_{D} \circ \varphi - c^{k}_2 \mathcal{X}_{D^c} \circ \varphi)^{2}\mathrm{d}\bm{x} \\
&+\alpha_{1}\int_{\Omega}|\Delta\varphi|^{2}\mathrm{d}\bm{x}+\alpha_{2}\int_{\Omega}\phi(|\mu(\varphi)|^{2})\mathrm{d}\bm{x}.
\end{aligned}
\end{aligned}\right.
\end{equation} 

For the first subproblem in \eqref{iterativescheme}, it has the closed-form solution
\begin{equation}\label{eq:c_subproblem}
c_{1}^{k} = \frac{\int_{G^{k-1}} I \mathrm{d}\bm{x}}{\int_{G^{k-1}}1\mathrm{d}\bm{x}} \quad \text{and} \quad
c_{2}^{k} = \frac{\int_{(G^{k-1})^{c}} I \mathrm{d}\bm{x}}{\int_{(G^{k-1})^{c}}1\mathrm{d}\bm{x}},
\end{equation}
where $G^{k-1} = \{\bm{x}\in \Omega| \varphi^{k-1}(\bm{x})\in D\}$ is the deformed $D$ under the transformation $\varphi^{k-1}$ and $(G^{k-1})^{c}$ is the complement of $G^{k-1}$. 

For the second subproblem in \eqref{iterativescheme}, by setting $J(\varphi) = c^{k}_1 \mathcal{X}_{D} \circ \varphi + c^{k}_2 \mathcal{X}_{D^c} \circ \varphi$, we can rewrite it as the following equivalent problem

 \begin{equation}\label{eq:registrationproblem}
\min_{\varphi}  \frac{1}{2}\int_{\Omega}(I^n - J(\varphi))^{2}\mathrm{d}\bm{x}+\alpha_{1}\int_{\Omega}|\Delta\varphi|^{2}\mathrm{d}\bm{x}+\alpha_{2}\int_{\Omega}\phi(|\mu(\varphi)|^{2})\mathrm{d}\bm{x}.
\end{equation}

Obviously, the problem \eqref{eq:registrationproblem} is a standard image registration problem, which can be solved by the first-discretize-then-optimize method, namely directly discretize the variational model by a proper discretization scheme to derive an unconstrained finite dimensional optimization problem and then choose a suitable optimization algorithm to solve the resulting unconstrained finite dimensional optimization. 

For the discretization, we employ the nodal grid to discretize the fitting term, smooth regularization term and Beltrami regularization term, whose discretized formulation can be found in \cite{zhang2018novel,zhang2022unifying,zhang2021topology}. Set $F(\Psi)$ as the discretized formulation of \eqref{eq:registrationproblem}. To solve the following optimization problem
\begin{equation}\label{optimizationproblem}
\min_{\Psi} F(\Psi),
\end{equation}
we choose the generalized Gauss-Newton method. We first solve the generalized Gauss-Newton equation
\begin{equation}\label{Gauss-Newtonequation}
\hat{H}^{l}p^{l} = -dF^{l},
\end{equation}
where $dF^{l}$ is the gradient of $F(\Psi)$ at $\Psi^{l}$ and $\hat{H}^{l}$ is the symmetric positive definite part of the hessian of $F(\Psi)$ at $\Psi^{l}$, to obtain the search direction $p^{l}$. Then we determine the step length by the Armijo strategy, simultaneously satisfying the energy sufficient descent condition and guaranteeing the Jacobian determinant of the discretized transformation $\Psi^{l+1}$ larger than $0$ \cite{zhang2018novel,zhang2021topology}. The stopping criteria is consistent with \cite{zhang2018novel,zhang2021topology}, namely when the change in the objective function, the norm of the update and the norm of the gradient are all sufficiently small, the iterations are terminated. The algorithm of the generalized Gauss-Newton method to solve \eqref{optimizationproblem} is summarized in Algorithm \ref{alg:generalizedGaussNewton}.

\begin{algorithm}[htbp!]
\caption{Generalized Gauss-Newton method to solve \eqref{optimizationproblem}.}
\label{alg:generalizedGaussNewton}
\begin{algorithmic}
\STATE{Give the discretized transformation $\Psi^{1}$ such that the Jacobian determinant of  $\Psi^{1}$ is larger than $0$;}
\FOR{$l=1:\text{MaxGN}$}
\STATE{solve \eqref{Gauss-Newtonequation} to obtain the search direction $p^{l}$;}
\STATE{Update $\Psi^{l+1}$ by Armijo Strategy;}
\IF{stopping criteria is satisfied}
\STATE{berak;}
\ENDIF
\STATE{Compute $\hat{H}^{l+1}$ and $dF^{l+1}$;}
\ENDFOR
\end{algorithmic}
\end{algorithm}

Following the proof of Theorem 5 in \cite{zhang2021topology}, we have the following convergence theorem for Algorithm \ref{alg:generalizedGaussNewton}.

\begin{theorem}
For the resulting finite-dimensional optimization problem \eqref{optimizationproblem}, given $\Psi^{1}$ satisfying the Jacobian determinant of $\Psi^{1}$ larger than $0$, the sequence $\{\Psi^{k}\}_{k\in\mathbb{N}}$ generated by Algorithm \ref{alg:generalizedGaussNewton} from $\Psi^{1}$ admits a subsequence that converges to a critical point $\Psi^{*}$ of $F$ and the Jacobian determinant of $\Psi^{*}$ larger than $0$.
\end{theorem}

Now we are ready to give Algorithm \ref{alg:algorithmalternatingdirection} to solve the problem \eqref{eq:eachstep_mainproblem}.

\begin{algorithm}[htbp!]
\caption{Alternating direction method to solve \eqref{eq:eachstep_mainproblem}.}
\label{alg:algorithmalternatingdirection}
\begin{algorithmic}
\STATE{Compute $M^{n}_{\mathrm{clicks}}$ by \eqref{eq:compute_clickmap} and $r^{n}$ by \eqref{eq:choose_positveclick} or \eqref{eq:choose_negativeclick}. Input target image $I^n = I^{n-1}+r^nM^{n}_{\mathrm{clicks}}$ and initial mask $D$ and set parameters $\alpha_{1}$ and $\alpha_{2}$. Give the discretized transformation $\Psi^{1}$ such that the Jacobian determinant of $\Psi^{1}$ is larger than $0$;}
\FOR{$k=1:\text{MaxAD}$}
\STATE{Update $c^{k+1}_{1},c^{k+1}_{2}$ by \eqref{eq:c_subproblem};}
\STATE{Update $\Psi^{k+1}$ by Algorithm \ref{alg:generalizedGaussNewton};}
\ENDFOR
\end{algorithmic}
\end{algorithm}

To further speed up Algorithm \ref{alg:algorithmalternatingdirection}, the multilevel strategy is often used in the implementation \cite{zhang2018novel,zhang2022unifying,zhang2021topology}. Firstly, we coarsen the target image $I^{n}$ by some levels. Then we can obtain a solution by solving the problem \eqref{eq:eachstep_mainproblem} on the coarsest level. Next, we interpolate the solution to the finer level as the initial guess for the next level. We repeat this process and get the final segmentation result on the finest level. The most important advantage of this strategy is that it can save computational time to provide a good initial guess for the finer level because there are fewer variables on the coarser level. Also, it can help to avoid to trapping into a local minimum since the coarser level only shows the main features and patterns.

\section{Experiments}\label{Experiments}

We have conducted extensive experiments on synthetic and real images to evaluate the performance of our proposed Quasi-conformal Interactive Segmentation (QIS) model. In this section, we will report the experimental results. All implementations were executed using Matlab R2022b on a Windows 11 x64 platform with a 3.20 GHz AMD Ryzen 5800H processor and 16 GB RAM. In our experiments, all images were resized into $512 \times 512$ and their intensities were rescaled into $[0, 255]$. The weight for the Laplacian term is $\alpha_1 = 0.001$, and the weight for the regularization of the Beltrami coefficients is $\alpha_2 = 100$. The implementation for the Chan-Vese model is from \cite{wu2024cv} with default parameters while the GrabCut is the built-in function of Matlab.

\subsection{Interactive segmentation of synthesized images} 
\begin{figure}[ht!]
    \centering
    \includegraphics[width=0.6\textwidth]{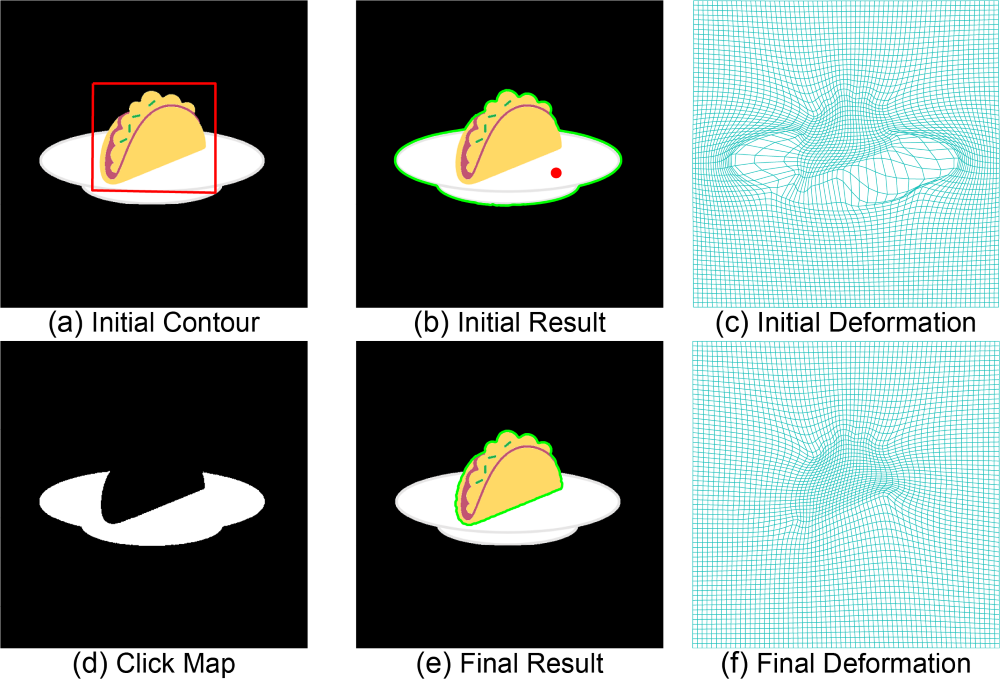}
    \caption{Segmentation result achieved by employing the QIS model on a synthesized image featuring a taco on a plate, with the objective of segmenting the taco.}
    \label{fig:clickcluster1}
\end{figure}
\begin{figure}[ht!]
    \centering
    \includegraphics[width=0.6\textwidth]{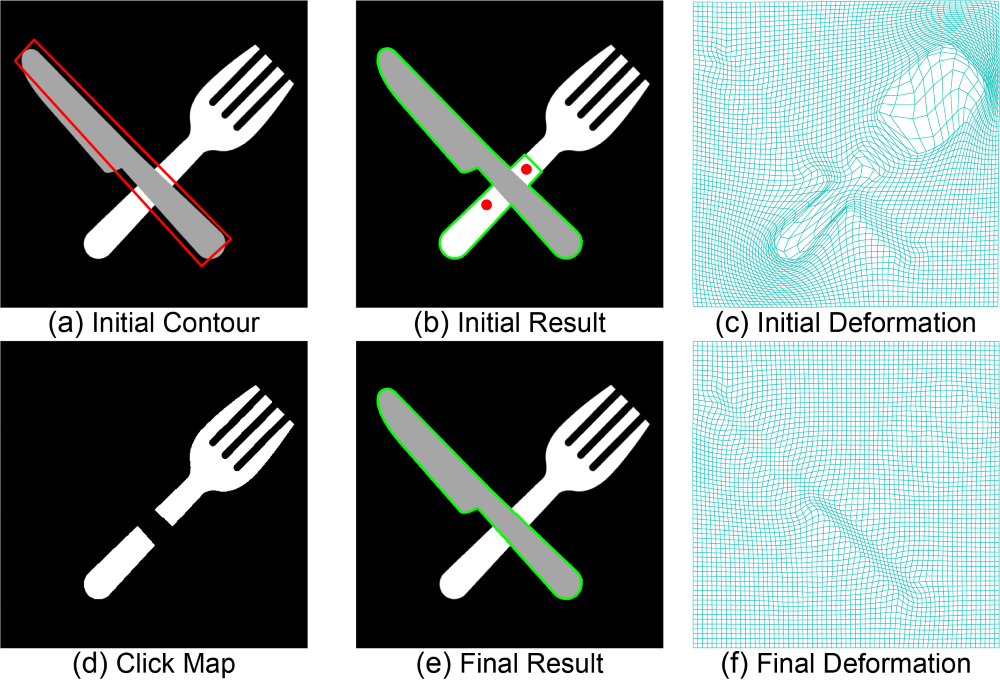}
    \caption{Segmentation result attained by utilizing the QIS model on a synthesized image highlighting the eating utensils, specifically the knife and fork, with the aim of segmenting the knife.}
    \label{fig:clickcluster2}
\end{figure}

We first evaluate our proposed model on simple synthetic images to provide a fundamental understanding of its functionality. To segment the taco in Figure \ref{fig:clickcluster1}, the model operates as follows. The user provides an initial contour (a) to approximate the region of interest and establish the desired topology of the segmentation mask. As described in Section \ref{sec:overallmodel}, the initial segmentation (b) is obtained by deforming the initial contour using the mapping shown in (c). This initial segmentation includes both the taco and the plate. To exclude the plate, the user inputs a negative click (indicated by the red point in (b)). Using our proposed method for computing click maps, detailed in Section \ref{sec:clickmap}, we obtain the updated click map in (d). With this click map and the previous segmentation result, we solve the interactive segmentation energy term \eqref{eq:step-n}. The final result is refined and improved as shown in (e), using the quasi-conformal mapping illustrated in (f).

The second example, illustrated in Figure \ref{fig:clickcluster2}, involves segmenting two eating utensils: a fork and a knife. Following the same procedures, the initial contour, initial segmentation, and associated deformation mapping are provided in (a), (b), and (c), respectively. However, since the regions to be excluded are interrupted by the knife, two negative clicks are required because each click can only indicate one connected region by the click map algorithm used (Section \ref{sec:clickmap}). Therefore, the user must input two clicks to properly exclude both disconnected regions. The computed click map is presented in (d). The final segmentation result and its associated quasi-conformal mapping are shown in (e) and (f), respectively.

\subsection{Interactive segmentation of medical images}
In this subsection, we evaluate our proposed segmentation model on real medical images. More specifically, we use brain magnetic resonance images (MRIs) from the BraTs21\cite{baid2021rsna} dataset in our experiment.

\begin{figure}[ht!]
    \centering
    \includegraphics[width=0.6\textwidth]{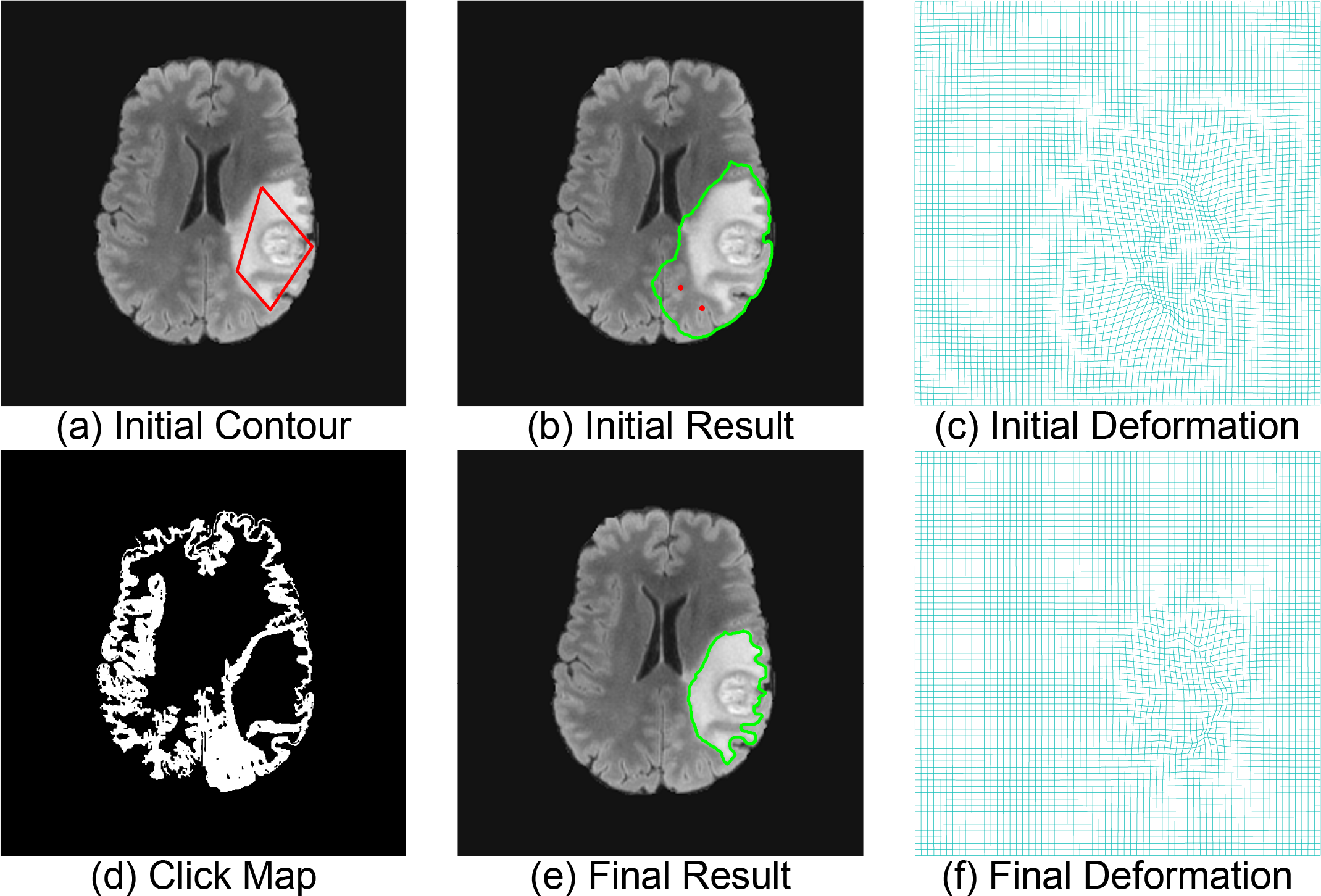}
    \caption{Segmentation result achieved by employing the QIS model on a brain MRI, with the goal of extracting the lesion region.}
    \label{fig:BISmed1}
\end{figure}
\begin{figure}[ht!]
    \centering
    \includegraphics[width=0.6\textwidth]{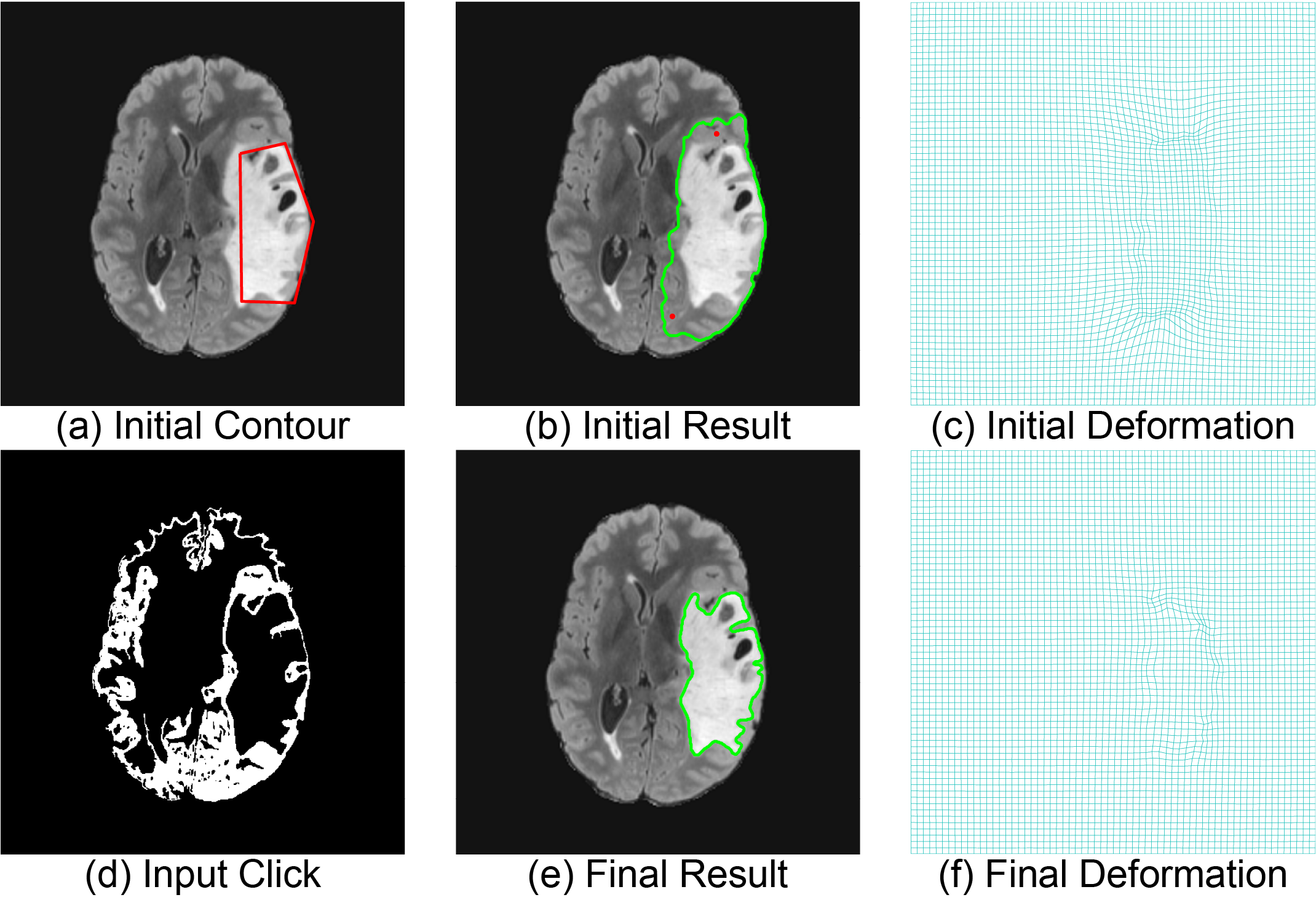}
    \caption{Another example of a segmentation result obtained by utilizing the QIS model on a different brain MRI, aiming to extract the lesion region.}
    \label{fig:BISmed2}
\end{figure}

Figure \ref{fig:BISmed1} shows an MRI of a brain with a lesion. Our objective is to segment the lesion area. Following the layout of the previous experiment, we illustrate the initial contour in (a). The initial segmentation and quasi-conformal mapping are shown in (b) and (c). From the image, it is observed that the lesion area is not connected. However, due to the topology-preserving property of the quasi-conformal mapping and the prescribed topology in the initial contour, the segmentation result is a simply-connected region. By applying two negative clicks and using the computed click map (d), we achieve the final segmentation result in (e) and its associated quasi-conformal mapping in (f).

Another example in Figure \ref{fig:BISmed2} depicts a more complex scenario with multiple holes inside the lesion regions. Those abnormal voids inside the area add to the segmentation complexity, often causing other methods to fail to include them. However, our method demonstrates robustness. By appropriately defining the initial contour and utilizing the topology-preserving features of the quasi-conformal mapping, our method effectively encloses a simply-connected lesion region that includes these abnormal areas, as shown in (e) and its associated mapping in (f).

\begin{figure}[ht!]
    \centering
    \includegraphics[width=0.6\textwidth]{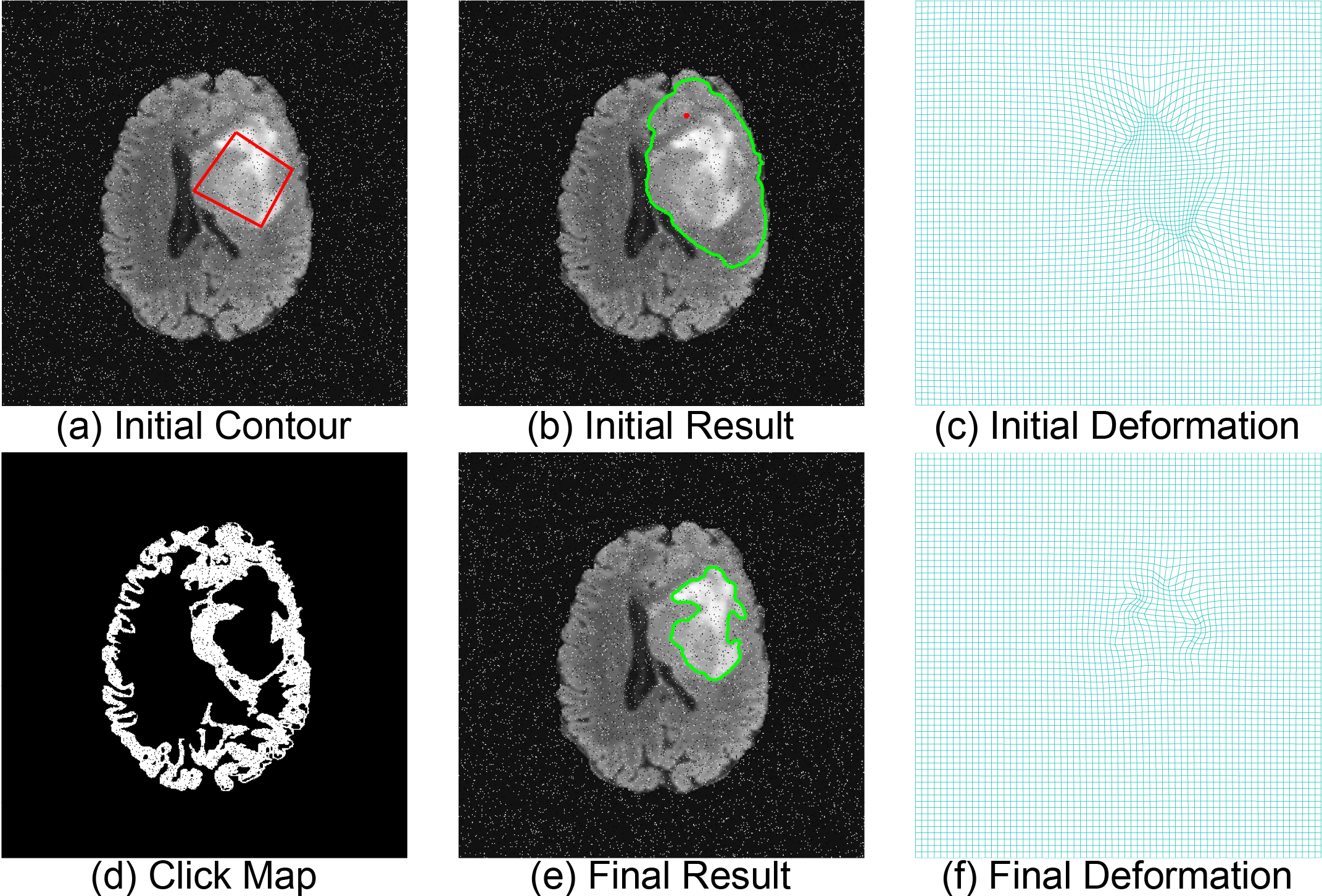}
    \caption{Segmentation result obtained by applying the proposed QIS model to extract the lesion region from a noisy brain MRI.}
    \label{fig:BISnoisy1}
\end{figure}
\begin{figure}[ht!]
    \centering
    \includegraphics[width=0.6\textwidth]{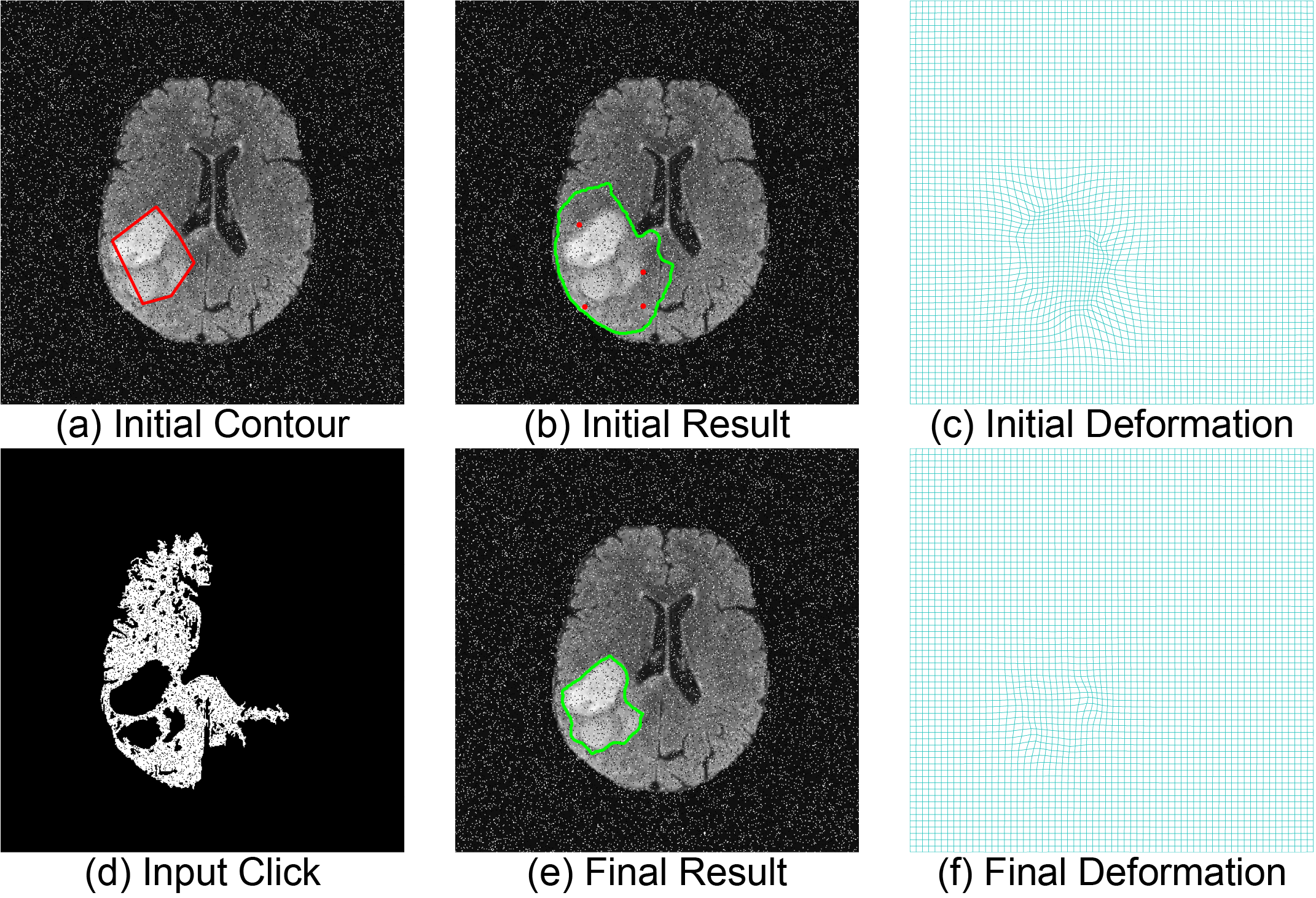}
    \caption{Another instance of a segmentation result obtained by applying the proposed QIS model to extract the lesion region from a different noisy brain MRI.}
    \label{fig:BISnoisy2}
\end{figure}

Figure \ref{fig:BISnoisy1} and Figure \ref{fig:BISnoisy2} show the even more challenging cases, which is to segment the lesion regions under brain MRIs with heavy noise. Following the same procedures, in this task, our methods successfully enclose the lesion region with user guidance. The result region is continuous and free of any outliers, corresponding to the prescribed topology from the initial contour input by the user. The user input clicks are effectively incorporated into the model to provide guidance for segmenting the target region. The deformation mappings are all bijective, revealing the topology-preserving feature of our interactive segmentation model.

\subsection{Comprehensive analysis of QIS model}
In this subsection, we perform a more comprehensive analysis of our proposed QIS model.

\paragraph{Comparison with interactive Chan-Vese segmentation model}

The QIS model is a deformation-based segmentation model incorporating the quasiconformal segmentation model with the click map. One advantage of the quasiconformal segmentation model is that it ensures the segmentation mask maintains the prescribed topology, avoiding topological errors and enhancing the accuracy and efficiency of the interactive segmentation model. To investigate the importance of incorporating the quasiconformal segmentation model, we compare our method with an alternative method that integrates the level-set segmentation model with the click map. Instead of utilizing the quasiconformal segmentation model, we can run the Chan-Vese segmentation model with $I^{n-1} + r^{n} M_{clicks}^{n}$ as the input image using a suitable parameter $r^{n}$, which we will refer to as interactive Chan-Vese model here after. The segmentation results of several natural images using the interactive Chan-Vese model are presented in Figure \ref{fig:BISnatCV}. The first row of Figure \ref{fig:BISnatCV} illustrates the initial segmentation outcomes of four natural images using the conventional Chan-Vese segmentation model. The two images on the left are clean, while the two on the right are noisy. The second row displays the initial segmentation results of the same four natural images using the quasiconformal segmentation model. In both methods, the initial segmentation masks include some undesired regions and exclude certain regions of interest. Some positive clicks (in blue) and negative clicks (in red) are introduced to guide the segmentation. The third row shows the segmentation results obtained by the interactive Chan-Vese segmentation method, incorporating click maps. The segmentation masks exhibit improved accuracy in segmenting the regions of interest, though they are not entirely precise. Finally, the last row demonstrates the segmentation results obtained by the QIS model, which accurately segments the regions of interest. The topology of the segmentation mask is also consistent with that of the initial contour. Note that the clicks inputted by users are the same for both the interactive Chan-Vese segmentation model and the QIS model.

Figure \ref{fig:BISmedCV} presents another comparison using brain MRIs. The first row of Figure \ref{fig:BISmedCV} shows the initial segmentation results of four brain MRIs using the conventional Chan-Vese segmentation model. Again, the two images on the left are clean, while the two on the right are noisy. Evidently, the initial segmentation masks contain some undesired regions and exclude some regions of interest. The second row shows the segmentation results obtained by the interactive Chan-Vese segmentation method, incorporating click maps. The clicks inputted by the users are the same as those used by the QIS model, as shown in Figure \ref{fig:BISmed1}, Figure \ref{fig:BISmed2}, Figure \ref{fig:BISnoisy1} and Figure \ref{fig:BISnoisy2}. The segmentation masks can better segment the lesion, although they are still not accurate. The last row shows the segmentation results obtained by the QIS model, which accurately segments the lesions. The topology of the segmentation mask is consistent with that of the initial contour. 

\begin{figure}[ht!]
    \centering
    \includegraphics[width=\textwidth]{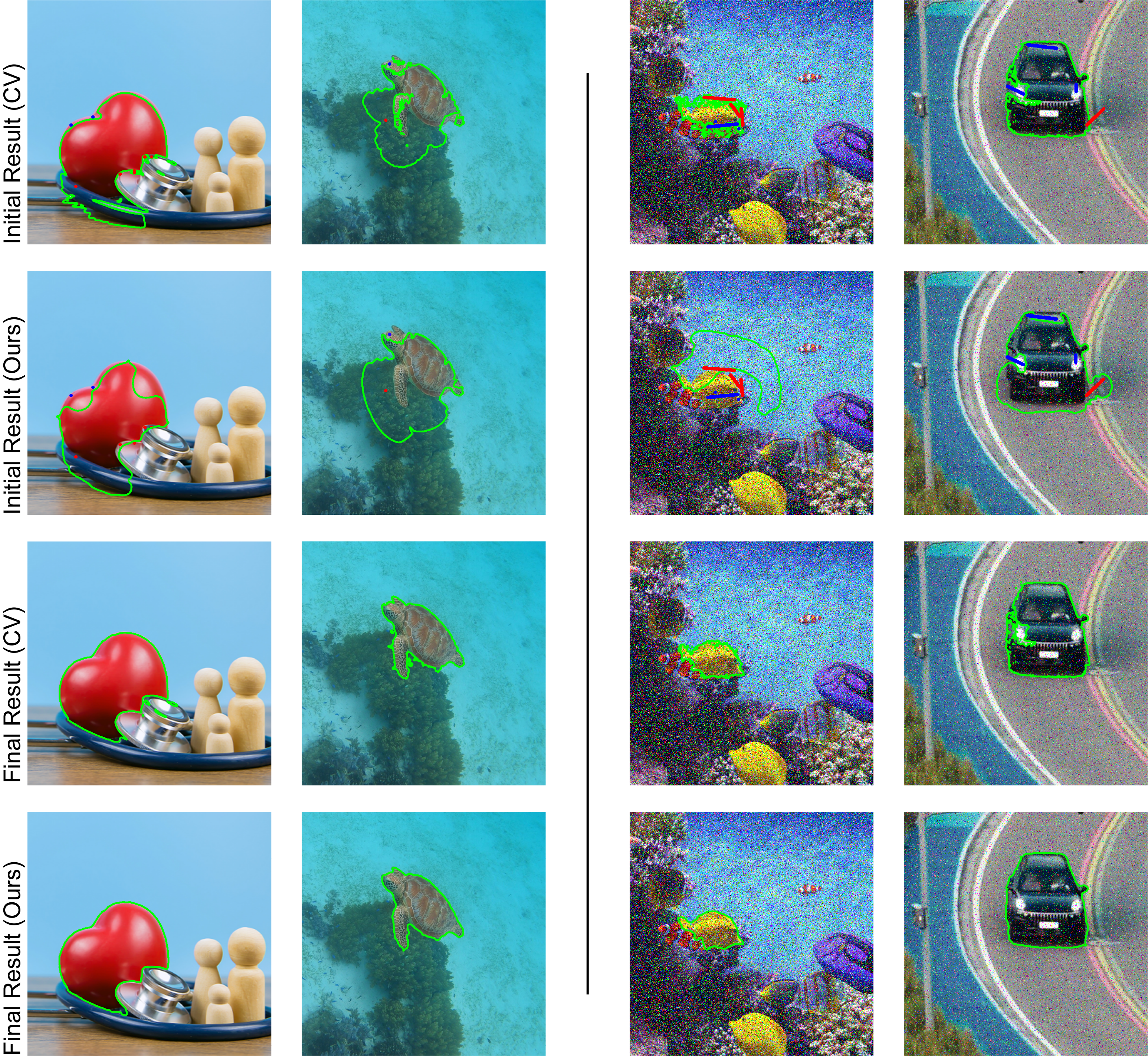}
    \caption{Segmentation results on natural images obtained using the QIS model and the interactive Chan-Vese segmentation model with the integration of click maps. The left side displays the segmentation results on two clean images, while the right side shows the segmentation results on two noisy images.}
    \label{fig:BISnatCV}
\end{figure}

\begin{figure}[ht!]
    \centering
    \includegraphics[width=\textwidth]{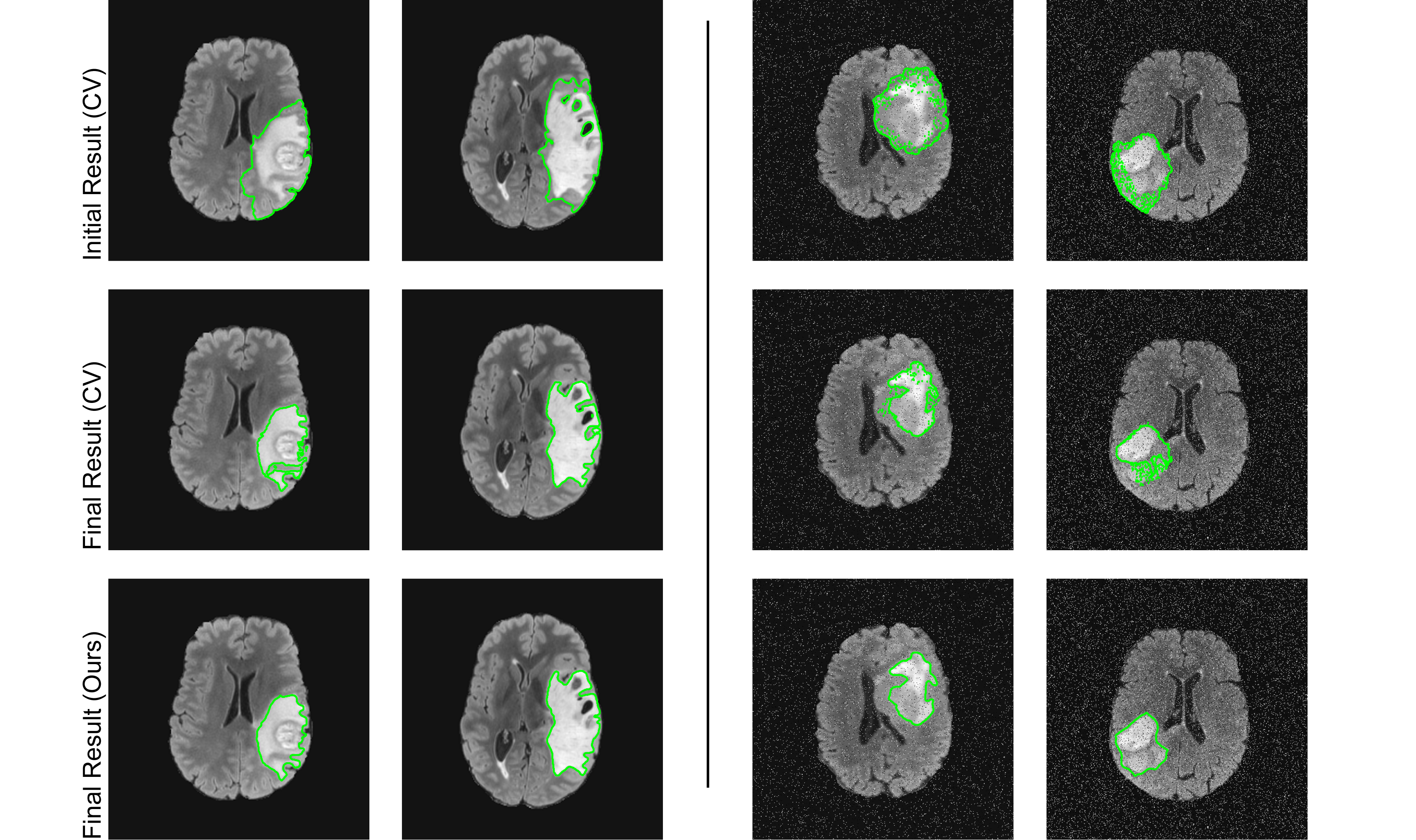}
    \caption{Segmentation results on brain MRIs obtained using the QIS model and the interactive Chan-Vese segmentation model with the integration of the click map. The left side displays the segmentation results on two clean brain MRIs, while the right side shows the segmentation results on two noisy brain MRIs.}
    \label{fig:BISmedCV}
\end{figure}

\paragraph{Comparison with GrabCut \cite{rother2004grabcut}} We compare the QIS model with a popular interactive segmentation model called GrabCut \cite{rother2004grabcut}. Similar to the QIS model, GrabCut is an interactive image segmentation algorithm that uses iterative graph cuts and Gaussian Mixture Models (GMMs) to separate foreground objects from the background. The user provides an initial bounding box around the object of interest, and the algorithm alternates between refining the segmentation using graph cuts and updating the color distribution models via GMMs. The comparison between the QIS and GrabCut methods on a noisy image capturing flowers is shown in Figure \ref{fig:BISnoise1}. In Figure \ref{fig:BISnoise1}(a), the iterative process of the QIS model is displayed. The negative clicks in each iteration, represented by red brushes, are also shown. The QIS effectively segments the noisy image in four steps with only a few brushes introduced. Figure \ref{fig:BISnoise1}(b) illustrates the iterative process of GrabCut, where the brushes introduced by the user (red for foreground and blue for background) are also shown. GrabCut produces a reasonably accurate segmentation mask after seven steps, although it is evident that it is comparatively less accurate than QIS. Figure \ref{fig:BISnoise2} shows a further comparison between GrabCut and QIS using another noisy image capturing a road scene, with the objective of segmenting the car. The iterative process of QIS is shown in (a), where the positive and negative clicks are denoted by red and blue brushes, respectively. QIS accurately segments the car in just five steps. (b) displays the iterative process using GrabCut, with the segmentation result after eight steps exhibited, which is comparatively less accurate than QIS.
\begin{figure}[ht!]
    \centering
    \includegraphics[width=\textwidth]{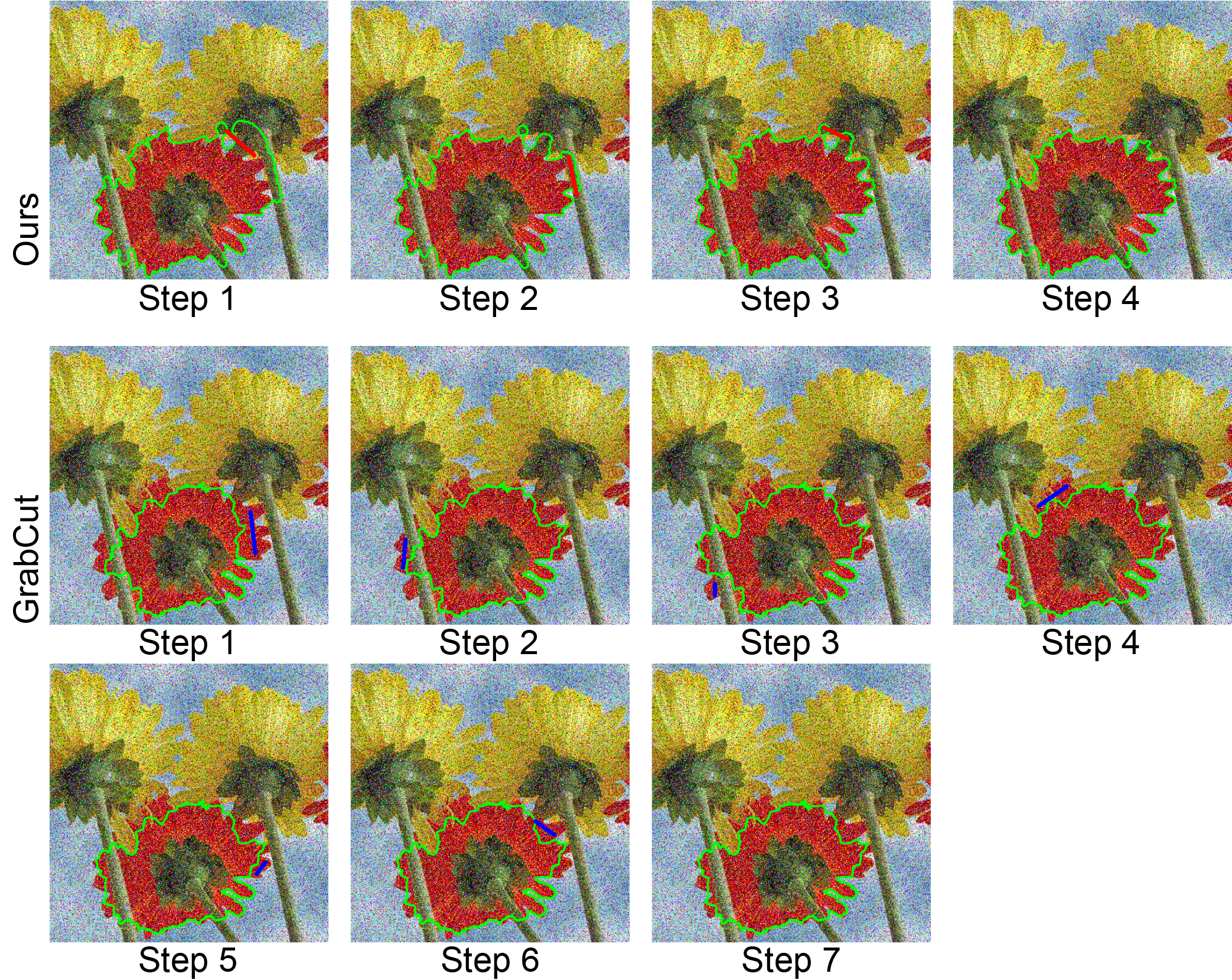}
    \caption{Comparison between QIS and GrabCut segmentation methods. The top shows the iterative segmentation process using the QIS method, with negative clicks labeled in red brushes. The final segmentation result is obtained in 4 steps. The bottom shows the iterative process using the GrabCut method, with positive and negative brushes labeled in blue and red, respectively. The segmentation result after 7 steps is displayed.}
    \label{fig:BISnoise1}
\end{figure}
\begin{figure}[ht!]
    \centering
    \includegraphics[width=\textwidth]{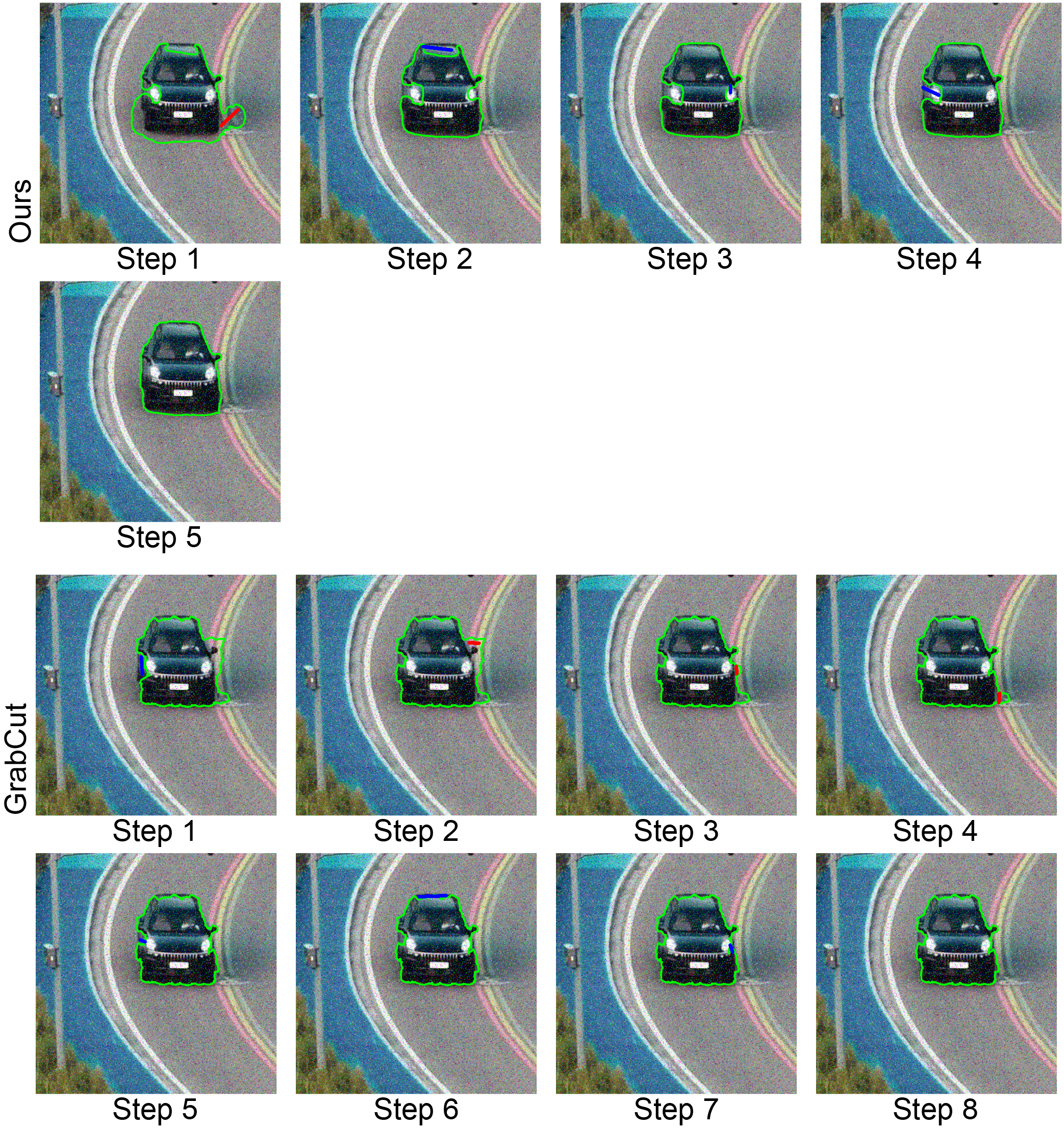}
    \caption{Another comparison between the QIS and GrabCut segmentation methods. The top demonstrates the iterative segmentation process using the QIS method, with positive and negative clicks labeled as blue and red brushes, respectively. The final segmentation result is achieved in 5 steps. The bottom depicts the iterative process using the GrabCut method, with positive and negative brushes labeled as blue and red, respectively. The segmentation results after 8 steps are displayed.}
    \label{fig:BISnoise2}
\end{figure}

\paragraph{Robustness across varied click choices} The QIS model relies on user inputs, specifically the positioning of positive and negative clicks. It is essential to examine the robustness of the QIS model when subjected to different click choices. Figure \ref{fig:BISvarious} depicts the robustness of the QIS model under various click choices. The initial segmentation mask contains undesired regions. In order to exclude these regions, different negative click options are introduced, as illustrated in the first row. The corresponding segmentation results produced by the QIS model, utilizing different click choices, are displayed in the second row. It is evident that the segmentation results remain consistent across different click choices, accurately delineating the lesion region. This exemplifies the robustness of the QIS model under different click choices.

\begin{figure}[ht!]
    \centering
    \includegraphics[width=\textwidth]{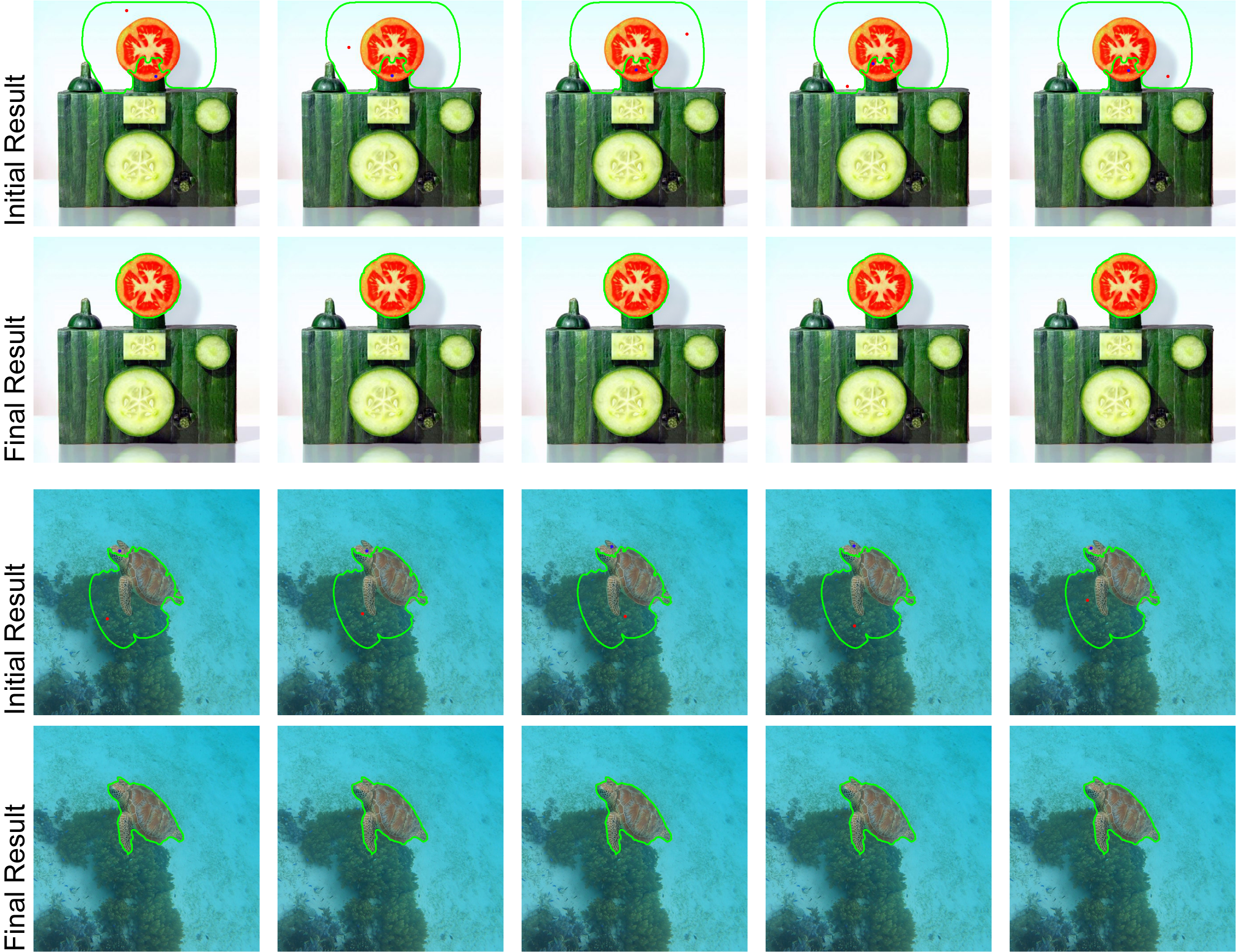}
    \caption{Segmentation results for extracting the region of interest obtained by the QIS model using different click choices. In each case, the upper row displays the initial segmentation results and input click positions, while the lower row presents the corresponding final segmentation results.}
    \label{fig:BISvarious}
\end{figure}

\section{Conclusion}
In this paper, we propose the quasiconformal interactive segmentation (QIS) model, with the incorporation of user interactions. User interactions are defined by two actions: a positive click and a negative click, represented by left and right clicks, respectively. These clicks serve to indicate the regions that the user wants to include or exclude from the segmentation process. Clicks are transformed into a click map, which is then integrated into the proposed model using a novel interactive segmentation energy term that we have designed. The segmentation mask is obtained by deforming a template mask with the same topology as the object of interest using an orientation-preserving quasiconformal mapping. This approach helps to avoid topological errors in the segmentation results. We provide a thorough analysis of the proposed model, including theoretical support for the ability of QIS to include or exclude regions of interest or disinterest based on the user's indication. To evaluate the performance of QIS, we conduct experiments on synthesized images, medical images, natural images, and noisy natural images. Regarding the computational efficiency, our method is not the fastest as presented in Appendix \ref{appendix:efficiency}. While our method may not be the fastest, as noted in Appendix \ref{appendix:efficiency}, it offers superior segmentation quality, representing a reasonable trade-off. Efforts are ongoing to further enhance its speed and efficiency.

Looking ahead, there are two key directions for future work. First, we aim to extend the QIS model to 3D image segmentation, enabling its application in volumetric data. Second, we plan to enhance segmentation outcomes by incorporating edge-aware energies, which can further improve the accuracy and robustness of the QIS model in challenging imaging scenarios.
\section*{Acknowledgments}
Lok Ming Lui is supported by HKRGC GRF (Project ID: 14307622). Daoping Zhang is supported by the National Natural Science Foundation of China (No. 12201320 and 12471484), the Fundamental Research Funds for the Central Universities, Nankai University (No. 63221039 and 63231144), and a visiting scholar of Chern Institute of Mathematics.

\bibliographystyle{siamplain}
\bibliography{references}

\appendix
 \section{Proof of Theorem \ref{theorem_three_value}}\label{proof1} 
\begin{proof}
    To prove this, we need to check if there is a case that a subset of $\Omega_0$, $\Omega_1$ or $\Omega_2$ is included in a local minimum. To do so, we investigate the value of $E(G = \mathcal{D}_0 \cup \mathcal{D}_1 \cup \mathcal{D}_2)$, where $\mathcal{D}_0 \subset \Omega_0$, $\mathcal{B}_0 = \Omega_0 \setminus \mathcal{D}_0$, $\mathcal{D}_1 \subset \Omega_1$, $\mathcal{B}_1 = \Omega_1 \setminus \mathcal{D}_1$, $\mathcal{D}_2 \subset \Omega_2$ and $\mathcal{B}_2 = \Omega_2 \setminus \mathcal{D}_2$.
    
    When $G=\mathcal{D}_0 \cup \mathcal{D}_1 \cup \mathcal{D}_2$, we have 
     \begin{equation*}
    c_1  = \frac{p_0||\mathcal{D}_0|| + p_1||\mathcal{D}_1|| + p_2||\mathcal{D}_2||}{||\mathcal{D}_0|| + ||\mathcal{D}_1|| + ||\mathcal{D}_2||}   
  \   \mathrm{and} \    
    c_2 = \frac{p_0 ||\mathcal{B}_0|| + p_1 ||\mathcal{B}_1|| + p_2 ||\mathcal{B}_2||}{||\Omega_0|| + ||\mathcal{B}_1|| + ||\mathcal{B}_2||}.
    \end{equation*}
    So we get 
    \begin{equation*}
    \begin{aligned}
    E(c_{1},c_{2},G=\mathcal{D}_0 \cup \mathcal{D}_1 \cup \mathcal{D}_2)  
    &= (p_0-c_1)^2 ||\mathcal{D}_0|| + (p_1-c_1)^2 ||\mathcal{D}_1|| + (p_2-c_1)^2 ||\mathcal{D}_2||\\ &\quad 
    + (p_0-c_2)^2 ||\mathcal{B}_0|| + (p_1-c_2)^2 ||\mathcal{B}_1|| + (p_2-c_2)^2 ||\mathcal{B}_2||.
    \end{aligned}
    \end{equation*}
    
    We take $A_0 = ||\Omega_0||$, $A_1 = ||\Omega_1||$, $A_2 = ||\Omega_2||$, $B_0 = ||\mathcal{B}_0||$, $D_0 = ||\mathcal{D}_0||$, $B_1 = ||\mathcal{B}_1||$, $D_1 = ||\mathcal{D}_1||$, $B_2 = ||\mathcal{B}_2||$, $D_2 = ||\mathcal{D}_2||$, $d_{12} = |p_1-p_2|$, $d_{01} = |p_0-p_1|$ and $d_{02} = |p_0-p_2|$. Then we can rewrite the above equation as:
    \begin{equation*}
    \begin{aligned}
    E(c_{1},c_{2},G)     
    &= \frac{d_{12}^2 D_1 D_2 + d_{01}^2 D_0 D_1 + d_{02}^2 D_0 D_2}{D_0 + D_1 + D_2}
    + \frac{d_{12}^2 B_1 B_2 + d_{01}^2 B_0 B_1 + d_{02}^2 B_0 B_2}{B_0 + B_1 + B_2}\\
    &=d_{12}^2 (D_2- \frac{D_0 D_2 + D_2^2}{D_0+A_1+D_2-B_1})
    + d_{01}^2 (D_0- \frac{D_0^2 + D_0 D_2}{D_0+A_1+D_2-B_1}) \\ 
    &\quad + d_{02}^2 \frac{D_0 D_2}{D_0+A_1+D_2-B_1}\\ 
    &\quad 
    + d_{12}^2 (B_2- \frac{B_0 B_2 + B_2^2}{B_0+B_1+B_2})
    + d_{01}^2 (B_0- \frac{B_0^2 + B_0 B_2}{B_0+B_1+B_2}) 
    + d_{02}^2 \frac{B_0 B_2}{B_0+B_1+B_2}. 
    \end{aligned}
    \end{equation*}  
 Treating $B_0$, $B_1$, $B_2$ as three variables, we can have $B_0 \in [0, A_0]$, $B_1 \in [0, A_1]$, and $B_2 \in [0, A_2]$ as $\mathcal{B}_0 \subset \Omega_0$, $\mathcal{B}_1 \subset \Omega_1$, and $\mathcal{B}_2 \subset \Omega_2$. Naturally, $D_0 = A_0 - B_0$, $D_1 = A_1 - B_1$, $D_2 = A_2 - B_2$. Further, define a function: $F(B_0,B_1,B_2) = E(c_{1},c_{2},G=\mathcal{D}_0 \cup \mathcal{D}_1 \cup \mathcal{D}_2).$  
    We then compute $\frac{\partial F}{\partial B_1}$:
    \begin{equation*}
    \begin{aligned}
    \frac{\partial F}{\partial B_1} 
    &= - \frac{d_{12}^2(D_0 D_2 + D_2^2)+d_{01}^2(D_0^2 + D_0 D_2)-d_{02}^2D_0 D_2}{(D_0+A_1+D_2-B_1)^2} \\ &\quad 
    +  \frac{d_{12}^2(B_0 B_2 + B_2^2)+d_{01}^2(B_0^2 + B_0 B_2)-d_{02}^2B_0 B_2}{(B_0+B_1+B_2)^2}.
    \label{eq:energydifferential}
    \end{aligned}
    \end{equation*}
    By definition of $d_{12}$, $d_{01}$ and $d_{02}$, we have $ (d_{01} + d_{12})^2 \geq d_{02}^2$ and $d_{01}^{2}+d_{12}^{2}\geq 2d_{01}d_{02}$.
    Thus, we obtain
    \begin{equation*}
     d_{12}^2(D_0 D_2 + D_2^2)+d_{01}^2(D_0^2 + D_0 D_2)-d_{02}^2D_0 D_2  \geq (d_{01}+d_{12})^{2}D_0 D_2 - d_{02}^2D_0 D_2 \geq 0
    \end{equation*}
    and
 \begin{equation*}
     d_{12}^2(B_0 B_2 + B_2^2)+d_{01}^2(B_0^2 + B_0 B_2)-d_{02}^2B_0 B_2 
    \geq 
    (d_{01}+d_{12})^{2}B_0 B_2 - d_{02}^2B_0 B_2 \geq 0.
    \end{equation*}
    Hence, both the first term and the second term of $\frac{\partial F}{\partial B_1}$ are monotonously decreasing for any $B_1 \in (0,A_1)$ given a fixed $B_0 \in (0,A_0)$ and $B_2 \in (0,A_2)$.

    As $\frac{\partial F}{\partial B_1}$ is monotonously decreasing for any given $B_0$ and $B_2$, there is at most one maximum within $B_1 \in (0,A_1)$ for any given $B_0$ and $B_2$. The minimum of $F(B_0,B_1,B_2)$ must only be obtained when $B_1 = 0$ or $B_1 = A_1$ for any given $B_0$ and $B_2$. So is for $B_2$ given any $B_0$ and $B_1$. Thus, the minimum can never be obtained when some subset of the domain $\Omega_i$ is included.
\end{proof}
 \section{Proof of Theorem \ref{theorem_positive_click}}\label{proof2} 
\begin{proof} 
Bringing the click map $M_{\mathrm{clicks}}=\mathcal{X}_{\Omega_2}$ into the interactive segmentation fidelity energy \eqref{eq:interactive}, we hope to get a segmentation $G^{**} = \Omega_1 \cup \Omega_2$ under a properly chosen $r$. Alternatively speaking, we hope that the following conditions are satisfied
\begin{equation}
\begin{cases}
    E_{int}(c_1, c_2, G = \Omega_1 \cup \Omega_2; M_{\mathrm{clicks}}=\mathcal{X}_{\Omega_2}) \leq E_{int}(c_1, c_2, G = \Omega_1; M_{\mathrm{clicks}}=\mathcal{X}_{\Omega_2}), \\
    E_{int}(c_1, c_2, G = \Omega_1 \cup \Omega_2; M_{\mathrm{clicks}}=\mathcal{X}_{\Omega_2}) \leq E_{int}(c_1, c_2, G = \Omega_2; M_{\mathrm{clicks}}=\mathcal{X}_{\Omega_2}).     
\end{cases}
\label{eq:H2condition}
\end{equation}
Note that for the computation of $E_{int}$, we can do some transformations as
\begin{equation}
\begin{split}
    E_{int}(c_1, c_2, G;M_{\mathrm{clicks}}) &= \int_\Omega (I + rM_{\mathrm{clicks}} - c_1 \mathcal{X}_{G} - c_2 \mathcal{X}_{G^c})^2 \mathrm{d}\bm{x} \\
    &= \int_\Omega (I^* - c_1 \mathcal{X}_{G} - c_2 \mathcal{X}_{G^c})^2 \mathrm{d}\bm{x},
    \label{eq:interactiveE}
\end{split}
\end{equation}
where 
\begin{equation}
I^* =
\begin{cases}
p_0, &\ (x,y) \in \Omega_0,\\
p_1, &\ (x,y) \in \Omega_1,\\
p_2 + r, &\ (x,y) \in \Omega_2.\\
\end{cases}    
\label{eq:imageandclick}
\end{equation}
The second equation in \eqref{eq:interactiveE} is similar to \eqref{eq:segmentationG}, then we can directly replace $p_{2}$ with $p_{2}+r$ in \eqref{eq:E1}, \eqref{eq:E2} and \eqref{eq:E12} to obtain
\begin{equation}
\begin{aligned}    
    &E_{int}(c_1, c_2, G=\Omega_1; M_{\mathrm{clicks}}=\mathcal{X}_{\Omega_2})= (p_0 - p_2 - r)^2 \frac{||\Omega_2|| ||\Omega_0||}{||\Omega_2|| + ||\Omega_0||},\\
    &E_{int}(c_1, c_2, G=\Omega_2; M_{\mathrm{clicks}}=\mathcal{X}_{\Omega_2}) = (p_0 - p_1)^2 \frac{||\Omega_1|| ||\Omega_0||}{||\Omega_1|| + ||\Omega_0||},\\
    &E_{int}(c_1, c_2, G=\Omega_1 \cup \Omega_2; M_{\mathrm{clicks}}=\mathcal{X}_{\Omega_2}) = (p_1 - p_2 - r)^2 \frac{||\Omega_1|| ||\Omega_2||}{||\Omega_1|| + ||\Omega_2||}.
\end{aligned}    
\label{eq:E_H2}
\end{equation}
Substituting \eqref{eq:E_H2} into \eqref{eq:H2condition}, we obtain
\begin{equation}\label{twoinequalities_1}
\begin{cases}
    (p_1-p_2-r)^2 \frac{||\Omega_1||}{||\Omega_1|| + ||\Omega_2||} \leq (p_0-p_2-r)^2\frac{||\Omega_0||}{||\Omega_2|| + ||\Omega_0||}, \\
    (p_1-p_2-r)^2 \frac{||\Omega_2||}{||\Omega_1|| + ||\Omega_2||} \leq (p_0-p_1)^2\frac{||\Omega_0||}{||\Omega_1|| + ||\Omega_0||}.
\end{cases}
\end{equation}

For the first inequality in \eqref{twoinequalities_1}, we have
$
 (p_1-p_2-r)^2 \frac{||\Omega_1||(||\Omega_2|| + ||\Omega_0||)}{||\Omega_0||(||\Omega_1|| + ||\Omega_2||)} \leq (p_0-p_2-r)^2. 
$
Taking $A=\sqrt{\frac{||\Omega_1||(||\Omega_2|| + ||\Omega_0||)}{||\Omega_0||(||\Omega_1|| + ||\Omega_2||)}}$, we get
\begin{equation}
        ((p_1-p_2)A - (p_0-p_2) - (A-1)r) ((p_1-p_2)A + (p_0-p_2) - (A+1)r) \leq 0. 
\label{eq:E_H2_IEQ1}
\end{equation}
For $A \neq 1$, we further obtain
\begin{equation*}
    (A-1)(A+1)\left(\frac{(p_1-p_2)A - (p_0-p_2)}{(A-1)} - r\right) \left(\frac{(p_1-p_2)A + (p_0-p_2)}{(A+1)} - r\right) \leq 0. 
\end{equation*}
Set $R_x = \frac{(p_1-p_2)A - (p_0-p_2)}{A-1}$ and $R_y = \frac{(p_1-p_2)A + (p_0-p_2)}{A+1}$. Note that here there can never be $p_1 = p_0$, or they should be in the same region for an object in the image. So we have
\begin{itemize}
\item
$
\begin{cases}
    R_x \leq r \leq R_y \ \text{when $p_0 > p_1$}\\
    R_y \leq r \leq R_x \ \text{when $p_0 < p_1$}\\
\end{cases}
$, when $A>1$ ($||\Omega_1|| > ||\Omega_0||$).
\item
$
\begin{cases}
    r\leq R_y \text{  or  } R_x\leq r \ \text{when $p_0 > p_1$}\\
    r\leq R_x \text{  or  } R_y\leq r \ \text{when $p_0 < p_1$}\\
\end{cases}
$, when $0<A<1$ ($||\Omega_1|| < ||\Omega_0||$).
\end{itemize} 
For $A=1$ ($||\Omega_1|| = ||\Omega_0||$), by \eqref{eq:E_H2_IEQ1}, we get
$
    (p_1 - p_0) \left(\frac{p_1+p_0- 2 p_2}{2} - r\right) \leq 0.
$
Since $R_y = \frac{p_1+p_0- 2 p_2}{2}$ when $A=1$, thus we have
\begin{equation*}
\begin{cases}
    R_y \geq r \ \text{when $p_0 > p_1$},\\
    R_y \leq r \ \text{when $p_0 < p_1$}.\\
\end{cases}
\end{equation*}

For the second inequality in \eqref{twoinequalities_1}, it can lead to  
$
    (p_1-p_2-r)^2 \frac{||\Omega_2||(||\Omega_1|| + ||\Omega_0||)}{||\Omega_0||(||\Omega_1|| + ||\Omega_2||)} \leq (p_0-p_1)^2.    
$
Taking $B=\sqrt{\frac{||\Omega_2||(||\Omega_1|| + ||\Omega_0||)}{||\Omega_0||(||\Omega_1|| + ||\Omega_2||)}}$, we have
$
    \left(\frac{(p_1-p_2)B + (p_0-p_1)}{B} - r\right) \left(\frac{(p_1-p_2)B - (p_0-p_1)}{B} - r\right) \leq 0.    
$
Set $Q_x = \frac{(p_1-p_2)B - (p_0-p_1)}{B}$ and $Q_y = \frac{(p_1-p_2)B + (p_0-p_1)}{B}$, then we get
\begin{equation*}
\begin{cases}
    Q_x \leq r \leq Q_y \ \text{when $p_0 > p_1$},\\
    Q_y \leq r \leq Q_x \ \text{when $p_0 < p_1$}.\\
\end{cases}
\end{equation*}
It is easy to find the relation among $R_x$, $Q_x$, $R_y$ and $Q_y$, 
\begin{equation*}
\begin{cases}
    R_x \leq Q_x \leq R_y \leq Q_y \ \text{when $A>1$ and $p_0 > p_1$},\\
    Q_y \leq R_y \leq Q_x \leq R_x \ \text{when $A>1$ and $p_0 < p_1$},\\
    Q_x \leq R_y \leq Q_y \leq R_x \ \text{when $1>A>0$ and $p_0 > p_1$},\\
    R_x \leq Q_y \leq R_y \leq Q_x \ \text{when $1>A>0$ and $p_0 < p_1$},\\
    Q_x \leq R_y \leq Q_y \ \text{when $A=1$ and $p_0 > p_1$},\\
    Q_y \leq R_y \leq Q_x \ \text{when $A=1$ and $p_0 < p_1$}.\\
\end{cases}
\end{equation*}

Hence, the model \eqref{eq:interactive} will give an optimal segmentation $G^{**}=\Omega_{1}\cup\Omega_{2}$ if the parameter $r$ satisfies the following condition
\begin{equation*}
\begin{cases}
    Q_x \leq r \leq R_y \ \text{when $p_0 > p_1$},\\
    R_y \leq r \leq Q_x \ \text{when $p_0 < p_1$},
\end{cases}
\label{neq:positive}
\end{equation*}
which completes the proof.
\end{proof}

\section{Proof of Theorem \ref{theorem_negative_click}}\label{proof3} 
\begin{proof}
Bringing the click map $M_{\mathrm{clicks}}=\mathcal{X}_{\Omega_2}$ into the interactive segmentation fidelity energy \eqref{eq:interactive}, we hope to get a segmentation $G^{**} = \Omega_1$ under a properly chosen $r$. Namely, the following conditions should be satisfied
\begin{equation*}
\begin{cases}
    E_{int}(c_1, c_2, G = \Omega_1; M_{\mathrm{clicks}}=\mathcal{X}_{\Omega_2}) \leq E_{int}(c_1, c_2, G = \Omega_1 \cup \Omega_2; M_{\mathrm{clicks}}=\mathcal{X}_{\Omega_2}), \\
    E_{int}(c_1, c_2, G = \Omega_1; M_{\mathrm{clicks}}=\mathcal{X}_{\Omega_2}) \leq E_{int}(c_1, c_2, G = \Omega_2; M_{\mathrm{clicks}}=\mathcal{X}_{\Omega_2}).     
\end{cases}
\label{eq:H12condition}
\end{equation*}

Similar to what we do in the positive click case, we also transform the interactive segmentation fidelity energy \eqref{eq:interactive} into \eqref{eq:interactiveE} and write $I^*$ as \eqref{eq:imageandclick}. With exactly the same consideration, we obtain

\begin{equation}\label{twoinequalities_2}
\begin{cases}
    (p_0-p_2-r)^2\frac{||\Omega_0||}{||\Omega_2|| + ||\Omega_0||} \leq (p_1-p_2-r)^2 \frac{||\Omega_1||}{||\Omega_1|| + ||\Omega_2||},\\
    (p_0-p_2-r)^2\frac{||\Omega_2||}{||\Omega_2|| + ||\Omega_0||} \leq (p_0-p_1)^2\frac{||\Omega_1||}{||\Omega_1|| + ||\Omega_0||}.
\end{cases}
\end{equation}

For the first inequality in \eqref{twoinequalities_2}, we have
$
    (p_0-p_2-r)^2 \frac{||\Omega_0||(||\Omega_2|| + ||\Omega_1||)}{||\Omega_1||(||\Omega_0|| + ||\Omega_2||)} \leq (p_1-p_2-r)^2.   
$
Then taking $C=\sqrt{\frac{||\Omega_0||(||\Omega_2|| + ||\Omega_1||)}{||\Omega_1||(||\Omega_0|| + ||\Omega_2||)}}$, $L_x = \frac{(p_0-p_2)C - (p_1-p_2)}{C-1}$ and $L_y = \frac{(p_0-p_2)C + (p_1-p_2)}{C+1}$ and following the same computational procedure in the positive click case, for $C\neq 1$, we get
\begin{itemize}
\item
$
\begin{cases}
    L_y\leq r \leq L_x \ \text{when $p_0 > p_1$}\\
    L_x\leq r \leq L_y \ \text{when $p_0 < p_1$}\\
\end{cases}
$, when $C>1$ ($||\Omega_0|| > ||\Omega_1||$).
\item 
$
\begin{cases}
    r\leq L_x \text{  or  } L_y\leq r \ \text{when $p_0 > p_1$}\\
    r\leq L_y \text{  or  } L_x\leq r \ \text{when $p_0 < p_1$}\\
\end{cases}
$, when $0<C<1$ ($||\Omega_0|| < ||\Omega_1||$).
\end{itemize}
For $C=1$ ($||\Omega_0|| = ||\Omega_1||$), we have
$
    (p_0 - p_1) \left(\frac{p_1+p_0- 2 p_2}{2} - r\right) \leq 0.
$
Since $L_y = \frac{p_1+p_0- 2 p_2}{2}$ when $C=1$, we thus obtain
\begin{equation*}
\begin{cases}
    L_y \leq r \ \text{when $p_0 > p_1$},\\
    L_y \geq r \ \text{when $p_0 < p_1$}.\\
\end{cases}
\end{equation*}

For the second inequality in \eqref{twoinequalities_2}, we have
$
    (p_0-p_2-r)^2 \frac{||\Omega_2||(||\Omega_0|| + ||\Omega_1||)}{||\Omega_1||(||\Omega_0|| + ||\Omega_2||)} \leq (p_0-p_1)^2.
$
Taking $E=\sqrt{\frac{||\Omega_2||(||\Omega_0|| + ||\Omega_1||)}{||\Omega_1||(||\Omega_0|| + ||\Omega_2||)}}$, then we get
$
    \left(\frac{(p_0-p_2)E - (p_0-p_1)}{E} - r\right) \left(\frac{(p_0-p_2)E + (p_0-p_1)}{E} - r\right) \leq 0 .
$
Set $P_x = \frac{(p_0-p_2)E - (p_0-p_1)}{E}$ and $P_y = \frac{(p_0-p_2)E + (p_0-p_1)}{E}$, so we obtain
\begin{equation*}
\begin{cases}
    P_x \leq r \leq P_y \ \text{when $p_0 > p_1$},\\
    P_y \leq r \leq P_x \ \text{when $p_0 < p_1$}.\\
\end{cases}
\end{equation*}
Through simple algebraic computation, we can find the relation among $L_x$, $P_x$, $L_y$ and $P_y$,
\begin{equation*}
\begin{cases}
    P_x \leq L_y \leq P_y \leq L_x \ \text{when $C>1$ and $p_0 > p_1$},\\
    L_x \leq P_y \leq L_y \leq P_x \ \text{when $C>1$ and $p_0 < p_1$},\\
    L_x \leq P_x \leq L_y \leq P_y \ \text{when $1>C>0$ and $p_0 > p_1$},\\
    P_y \leq L_y \leq P_x \leq L_x \ \text{when $1>C>0$ and $p_0 < p_1$},\\
    P_x \leq L_y \leq P_y \ \text{when $C=1$ and $p_0 > p_1$},\\
    P_y \leq L_y \leq P_x \ \text{when $C=1$ and $p_0 < p_1$}.\\
\end{cases}
\end{equation*}

Hence, the model \eqref{eq:interactive} will give an optimal segmentation $G^{**}=\Omega_{1}$ if the parameter $r$ satisfies the following condition
\begin{equation*}
\begin{cases}
    L_y \leq r \leq P_y \ \text{when $p_0 > p_1$},\\
    P_y \leq r \leq L_y \ \text{when $p_0 < p_1$}.
\end{cases}
\label{neq:negative}
\end{equation*}
which completes the proof.
\end{proof}

\section{Computational Efficiency Comparison}
\label{appendix:efficiency}
By solving a quasi-conformal registration-based segmentation model in each step, our method can preserve the topology and produce results without noise and outliers. However, as a deformable model, our method needs to solve a more complex problem compared to either the CV model or GrabCut, which leads to more time consumption, as reported in Table~\ref{tb:CompareCV} and Table~\ref{tb:CompGrabCut}.

\newcolumntype{D}{>{\centering\arraybackslash}p{1.1cm}}
\begin{table}[!ht]
\centering
\begin{tabular}{l|DD|DD|DD|DD}
\toprule
Experiment      & \multicolumn{2}{c|}{Heart}    & \multicolumn{2}{c|}{Turtle}   & \multicolumn{2}{c|}{Fish} & \multicolumn{2}{c}{Car}  \\\hline
Method          & QIS   & intCV                 & QIS   & intCV                 & QIS   & intCV             & QIS   & intCV             \\\hline
Ave. time /s    & 11.86 & 3.24                  & 13.17 & 3.72                  & 12.59 & 3.86              & 12.94 & 3.98              \\
No. of Clicks   & 8     & 8                     & 2     & 2                     & 6     & 6                 & 4     & 4                 \\
Total time /s   & 94.88 & 25.92                 & 26.34 & 7.44                  & 75.54 & 23.16             & 51.76 & 15.92             \\\hline
\bottomrule
\end{tabular}
\caption{\han{Comparison of computational efficiency between our method and the interactive Chan-Vese (iCV) model. The time reported for each step is the total time averaged over the number of steps. In this experiment, the number and placement of user clicks are kept identical for both methods across the various problems.}}
\label{tb:CompareCV}
\end{table}

\newcolumntype{C}{>{\centering\arraybackslash}p{1.5cm}}
\begin{table}[!ht]
\centering
\begin{tabular}{l|CC|CC}
\toprule
Experiment      & \multicolumn{2}{c|}{Flower}   & \multicolumn{2}{c}{Car}   \\\hline
Method          & QIS   & GrabCut               & QIS   & GrabCut           \\\hline
Ave. time /s    & 11.67 & 4.17                  & 11.27 & 4.23              \\
No. of Clicks   & 4     & 7                     & 5     & 8                 \\
Total time /s   & 46.68 & 29.19                 & 56.35 & 33.84             \\\hline
\bottomrule
\end{tabular}
\caption{\han{Comparison of computational efficiency between our method and GrabCut \cite{rother2004grabcut}. The time for each step is the total time averaged over the number of steps.}}
\label{tb:CompGrabCut}
\end{table}

\end{document}